\documentclass[journal]{IEEEtran}
\usepackage{cite}
\usepackage{amsmath,amssymb,amsfonts}
\usepackage{algorithmic}
\usepackage{graphicx,color}
\usepackage{textcomp}
\usepackage{xcolor}
\usepackage{hyperref}
\hypersetup{hidelinks=true}
\usepackage{algorithm,algorithmic}
\def\BibTeX{{\rm B\kern-.05em{\sc i\kern-.025em b}\kern-.08em
    T\kern-.1667em\lower.7ex\hbox{E}\kern-.125emX}}

\usepackage{amsthm}

\newtheorem{theorem}{Theorem}
\newtheorem{cor}{Corollary}
\newtheorem{remark}{Remark}
\usepackage[caption=false,font=footnotesize]{subfig}
\newtheorem{lemma}{Lemma}
\usepackage{orcidlink}

\begin{document}

\title{Cluster-Aware Over-the-Air Federated Learning with Energy-Harvesting Devices: From Global Training to Model Personalization}

% \author{First A. Author\authorrefmark{1}, Fellow, IEEE, Second B. Author\authorrefmark{2},\\ and Third C. Author Jr.\authorrefmark{3}, Member, IEEE}

% \author{Furkan Bagci\,\orcidlink{0009-0008-0473-5764}, Busra Tegin\,\orcidlink{0000-0002-3342-5414} (Member, IEEE),\\Mohammad Kazemi\,\orcidlink{0000-0001-5177-1874} (Member, IEEE), and Tolga M. Duman\,\orcidlink{0000-0002-5187-8660} (Fellow, IEEE)}

\author{Furkan Bagci, Busra Tegin, Mohammad Kazemi, and Tolga M. Duman

\thanks{

This work is funded by TUBITAK through the CHIST-ERA
project SONATA (CHIST-ERA-20-SICT-004, funded by TUBITAK, Turkey
Grant 221N366). Furkan Bagci's work was also supported by Türk Telekomünikasyon A.S. within the framework of the 5G and Beyond Joint Graduate Support Programme coordinated by the Information and Communication Technologies Authority. Mohammad Kazemi's work is supported by UKRI under the U.K. government's Horizon Europe Funding Guarantee under Grant 101103430.
Part of this work was conducted while Furkan Bagci and Busra Tegin were affiliated with Bilkent University. An earlier version of this paper
was presented in part at the 2025 IEEE International Conference on Communications (ICC) Workshops \cite{bagci2025update}. }

\thanks{ 

Furkan Bagci is with the Department of Electrical and Computer Engineering, University of Illinois Chicago, Chicago, IL 60607 USA (e-mail: fbagc@uic.edu).

Busra Tegin is with IETR-UMR CNRS 6164, CentraleSupélec Rennes Campus, Cesson-Sévigné 35576, France (e-mail: busra.tegin@centralesupelec.fr). 

Mohammad Kazemi is with the Department of Electrical and Electronic Engineering, Imperial College London, London SW7 2BT, U.K. (e-mail: mohammad.kazemi@imperial.ac.uk). 

Tolga M. Duman is with the Department of Electrical and Electronics Engineering, Bilkent University, Ankara 06800, Türkiye (e-mail: duman@ee.bilkent.edu.tr).
}

}


\maketitle

\begin{abstract}
Federated learning (FL) enables distributed optimization and learning across decentralized edge devices while preserving data privacy, but its performance is fundamentally constrained by heterogeneous data distributions, limited communication resources, and energy availability. In practical wireless networks, mobile devices (MDs) often exhibit diverse data and learning objectives, naturally forming clusters of users with jointly trainable models. When devices rely on energy harvesting (EH), stochastic energy arrivals further complicate participation and scheduling under communication constraints. In this work, we study over-the-air (OTA) FL with EH MDs under heterogeneous data distributions, and investigate two closely related learning objectives within a unified framework: one aiming for a more representative global model by reducing data bias, and the other learning more personalized cluster-specific models by exploiting this bias. 
% Both problems originate from the same clustered heterogeneous structure and are jointly shaped by the communication and energy limitations of wireless OTA FL systems. 
In the global training mode, cluster information guides energy- and diversity-aware scheduling, ensuring that the scheduled active users provide a more representative aggregate update. In the personalization mode, the same cluster structure defines cluster-level learning objectives and OTA recovery targets, enabling the parameter server to train multiple cluster-specific models through simultaneous transmissions over the wireless multiple-access channel. Numerical results demonstrate that the proposed unified framework improves fairness or personalization, depending on the operating mode, while reducing communication overhead.
\end{abstract}

\begin{IEEEkeywords}
Federated learning, over-the-air computation, energy harvesting, personalized learning, diverse sampling.
\end{IEEEkeywords}

% \section{Introduction}
\section{INTRODUCTION}

Federated Learning (FL) is a decentralized machine learning paradigm in which multiple users collaboratively train a shared model under central server coordination, without sharing their local data. Initially introduced in \cite{mcmahan2017commef}, FL offers key advantages, including improved data privacy, lower latency, and enhanced model quality by leveraging diverse user data while keeping it on-device \cite{konecny2017flstrat, nguyen2021FLIot}. One of the main challenges in real-world deployments is the limited bandwidth of the communication channel between devices and the parameter server (PS). Consequently, reducing communication overhead is crucial for the efficient and scalable implementation of FL. To address this issue, over-the-air (OTA) computation has been proposed, leveraging the superposition property of the wireless multiple-access channel (MAC) to enable simultaneous aggregation of model updates during transmission \cite{yang2020ota, amiri2020dsgd}.

Several studies have utilized OTA aggregation and proposed solutions to tackle practical challenges in real-world FL deployments. 
% For instance, \cite{sun2022dynamic} introduces an energy-aware dynamic scheduling algorithm for OTA FL systems that employs analog gradient aggregation, whereas digital aggregation for users with random participation is investigated in \cite{Tarizzo2025}. 
For instance, \cite{sun2022dynamic} proposes an energy-aware dynamic scheduling algorithm for OTA FL with analog gradient aggregation, while \cite{Tarizzo2025} studies digital aggregation under random user participation.
Another study \cite{bereyhi2023device} formulates the device scheduling problem as a sparse support selection task, aiming to select a subset of devices under a combined cost constraint. Also, in \cite{kim2023beamforming}, the authors address beamforming vector design and device scheduling for OTA FL. 
In many OTA FL settings, mobile devices (MDs) are often assumed to be blind, i.e., lacking channel state information (CSI), whereas the PS is assumed to possess full or partial CSI for aggregation. Even under such asymmetric CSI assumptions, prior work shows that accurate model aggregation is achievable \cite{amiri2020BFL, tegin2021blind, ozan2024}.

The deployment of FL faces many challenges, such as hardware impairments, data distribution issues, and limited energy for mobile devices. Energy-harvesting (EH) devices, which continuously collect energy from the environment, are increasingly adopted in wireless communication systems to alleviate energy constraints. Various studies examine the capacity of the EH channel across different battery sizes \cite{sennur2015EH}. 

FL for EH devices has attracted growing interest, with recent works addressing it through data-utility–aware subcarrier allocation \cite{zeng2024ehd} and energy- and channel-aware device scheduling \cite{hamdi2022ehwn}.
OTA FL with EH devices has been studied under various settings, including joint user selection and receive beamforming \cite{chen2023joint}, energy- and arrival-aware weighted aggregation \cite{ozan2022ehfl}, online transceiver design for dynamic scheduling \cite{an2024online}, and MDP-based approaches for joint scheduling and power control \cite{zhang2024mdp}. 
However, existing approaches often overlook the impact of user data distributions in energy-harvesting OTA FL systems operating over wireless fading channels.

Studies on FL with non-independent and identically distributed (non-i.i.d.) data show that heterogeneous local datasets significantly impact model accuracy and convergence \cite{zhao2018noniid, li2019convergence}. 
These highly non-i.i.d. data distributions can cause users and the system to drift toward different, often conflicting objectives, naturally leading to the formation of clusters with similar data characteristics and learning goals.
The key challenge is that energy availability and data heterogeneity interact. EH constraints determine which users can participate, while non-i.i.d. data determines whether the set of participants is representative of all users. In OTA FL, this issue is further complicated because the PS observes only superimposed updates rather than individual client updates. This motivates the use of a user-cluster structure as a compact representation of user diversity, which can guide energy-aware scheduling when a single global model is desired and define cluster-level learning objectives when personalization is needed.

Several studies have explored how to leverage inter-user relationships to improve convergence and performance in federated learning. 
For global training, the authors in \cite{fraboni2021clusteredsampling} propose clustered sampling, a user selection strategy that groups users based on the cosine similarity of their updates and samples a diverse subset across clusters. This approach encourages representative participation and leads to smoother convergence. Similarly, federated averaging with diverse user selection (DivFL) selects a small but diverse subset of users in each round to reduce communication overhead while approximating the effect of global aggregation \cite{balakrishnan2022diverse}. 
The aforementioned studies are based on the premise that users with similar data distributions tend to produce similar model updates, thereby enabling inference of characteristics of the underlying local data from these updates. While this similarity can potentially lead to privacy leakage, it can also be leveraged to improve model performance and accelerate convergence in FL.

The same inter-user similarity has also motivated clustered and personalized FL methods. 
In clustered federated learning (CFL), users with aligned learning objectives are partitioned into clusters, and separate models are trained for each cluster rather than a single global model \cite{Sattler2021CFL, Ghosh2022CFL}. 
This improves personalization under data heterogeneity, but it also introduces additional communication and computation costs, especially when multiple cluster models must be trained and updated separately. Hence, CFL in wireless communication setups is gaining traction as a promising research direction, particularly in light of the challenges posed by data heterogeneity and limited communication resources  \cite{Sami2024CFL, Lin2024CFL, Li2025PFL}. In \cite{Sami2024CFL}, a clustered OTA FL framework is proposed that leverages beamforming and gradient compression to train multiple models. Meanwhile, \cite{Lin2024CFL} introduces a bi-level optimization approach to jointly optimize global and personalized local models. Additionally, \cite{yao2025multitask} presents a hybrid beamforming framework for OTA multitask FL that aims to learn task-specific models.

Existing works typically exploit user similarity for a specific learning objective: diversity-based client selection methods use it to improve the representativeness of the participating users for single-model training, whereas clustered and personalized FL methods use it to train separate models for different user groups. 
However, these two roles of user similarity have largely been studied separately. 
In contrast, we study user-cluster structure as a common abstraction for EH OTA-FL and show how it can be used in two complementary operating modes depending on the system objective.
When a single global model is desired, the cluster structure guides energy-aware and diversity-aware scheduling to reduce participation bias. 
When personalization is desired, the same cluster structure defines cluster-level learning objectives and OTA recovery targets, enabling simultaneous training of multiple cluster-specific models exploiting the data bias. 
This perspective is particularly important in OTA-FL with EH devices, where intermittent participation, superimposed update observations, limited CSI, and communication efficiency must be handled jointly.

In this paper, we develop a unified cluster-aware OTA-FL framework for EH mobile devices with highly heterogeneous data distributions. The key idea is to treat the cluster structure as a common mechanism that can support different learning objectives under the same wireless and energy-limited setting. When the system objective is to train a single global model, the cluster structure is used to guide diversity-aware scheduling. When the objective is personalization, the same cluster structure is used to define cluster-level learning objectives and OTA recovery targets, enabling the PS to train multiple cluster-specific models through simultaneous transmissions over the wireless MAC. 
For the global model setting, we propose scheduling methods that leverage users’ data characteristics to reduce redundant transmissions and improve energy efficiency. These methods exploit inter-user relationships induced by underlying data distributions.
Initially, we assume that user data distributions are known at the server and show that incorporating this information, along with users’ battery levels, into the scheduling strategy, referred to as entropy-based scheduling, can enhance global learning performance and reduce system overhead. We then show that, even without explicit knowledge of user data distributions, the PS can infer user characteristics from global updates using the proposed least-squares-based user representation estimation approach.
The estimation process does not recover the exact updates; instead, it yields a representation that reflects the overall direction of each user’s updates across multiple iterations. These inferred patterns can then be used to guide effective user scheduling and improve overall learning efficiency.

In the personalized setting, where users prioritize models that better reflect their own data distributions rather than a single global model averaged across all devices, which may be suboptimal for certain users, we aim to provide more personalized models for each cluster.
Accordingly, we study CFL approaches in the OTA FL setup, in the presence of EH MDs with different goals or interests. The primary goal is to train a dedicated model for each cluster to better capture their specific preferences, referred to as personalized clustered federated learning, while simultaneously serving all clusters using a single parameter server through OTA transmission to ensure communication efficiency. To enable simultaneous serving across clusters, we propose different combining methods at the server side, each tailored to a specific level of available CSI. 
The proposed CFL framework enables reliable cluster-level model recovery and enhances personalization performance under practical OTA and energy-harvesting constraints.

Our main contributions are summarized as follows:
\begin{itemize}
\item We study over-the-air federated learning with energy-harvesting mobile devices under highly heterogeneous data distributions, and develop a unified cluster-aware framework that either mitigates heterogeneity-induced bias or exploits heterogeneity for personalized learning, depending on system objectives.
%and derive cluster-aware strategies to mitigate bias caused by partial participation and limited energy availability.

\item We provide a convergence analysis for the proposed OTA-FL 
framework with EH devices, characterizing the limiting error under 
decaying stepsizes and the steady-state error floor under constant 
stepsizes. Both results explicitly decompose the bound into 
contributions from OTA channel noise, data heterogeneity, and 
gradient approximation error, providing a principled theoretical 
foundation for the scheduling and combining strategies proposed 
subsequently.

\item For the global model training, we propose diversity-aware user scheduling methods that exploit user data characteristics and energy states. Specifically, we develop entropy-based scheduling for known data distributions and a least-squares–based inference framework that estimates user representations directly from aggregated OTA signals when data distributions are unknown.

\item For personalized training, we introduce a personalized clustered federated learning framework that trains separate models for naturally aligned user clusters and enables simultaneous OTA aggregation across clusters using a single parameter server. We develop combining techniques tailored to different levels of CSI, supporting communication-efficient multi-model training with blind transmitters.

\item Extensive simulations on MNIST, FMNIST, and CIFAR-10 demonstrate that the proposed cluster-aware scheduling and OTA-based CFL approaches significantly improve fairness or personalization while reducing communication overhead in EH wireless FL systems.
\end{itemize}

The paper is organized as follows. Section \ref{sec2} introduces the OTA FL setup with EH MDs. Section \ref{sec3} presents the proposed update estimation and scheduling policies to achieve unbiased and diverse client participation in FL under both known and unknown data distributions. In Section \ref{sec4}, we first introduce a clustered personalized FL approach to train multiple personalized models based on user preferences, and then extend it to different levels of CSI. Section \ref{sec5} provides numerical results, and Section \ref{conc} concludes the paper.

\textit{Notations:} For vectors $\mathbf{x}$ and $\mathbf{y}$ of equal dimension, $\mathbf{x} \circ \mathbf{y}$ denotes the element-wise (Hadamard) product. We define $[i] \triangleq \{1, \dots, i\}$. The notation $\|\cdot\|_F$ represents the Frobenius norm, and $\operatorname{Tr}(\cdot)$ denotes the trace operator.

% ------------------------------------------------------------
% ------------------------------------------------------------
% ------------------------------------------------------------

% ------------------------------------------------------------
% ------------------------------------------------------------
% ------------------------------------------------------------

% \section{System Model and Preliminaries}
\section{\MakeUppercase{System Model and Preliminaries}}
\label{sec2}

We begin by describing the general system model used throughout this paper, which serves as the basis for the subsequent analyses and proposed methods. We consider an OTA FL setup with $M$ EH single-antenna MDs, each holding locally distributed and potentially heterogeneous datasets. The devices transmit their local model updates to a central parameter server through a fading multiple-access channel. It is assumed that there is no CSI at the transmitters (CSIT), while the PS, equipped with $K$ antennas, uses aggregated channel information to coherently align and recover the received signals, exploiting the OTA aggregation property of the MAC channel.

In FL, the primary objective is to minimize a global loss function, denoted as $F(\boldsymbol{\theta })$, collaboratively across $M$ devices, where $\boldsymbol{\theta } \in \mathbb{R}^{2N}$ represents the parameters of the global model to be optimized. The global loss function is defined as
\begin{equation}\label{gl_loss}
F(\boldsymbol{\theta }) =  \sum_{m=1}^{M} \frac{\left|\mathcal{B}_{m}\right|}{B} F_{m}(\boldsymbol{\theta}),
\end{equation}
where $\boldsymbol{\theta}$ represents the global model parameters, $\mathcal{B}_{m}$ is the local dataset of the $m$-th user, $m \in [M] $, where $[M] \triangleq \{1, \dots, M\}$, and $ B \triangleq \sum_{m=1}^{M} \left| \mathcal{B}_{m} \right| $. Also, $F_{m}(\boldsymbol{\theta})$ represents the average empirical local loss of the $m$-th user, which is
\begin{equation} \label{lc_loss}
    F_{m} \left ({\boldsymbol {\theta } }\right) = \frac {1}{\left| \mathcal{B}_{m} \right|} \sum \limits _{{u} \in \mathcal {B}_{m}} f \left ({\boldsymbol {\theta }, \boldsymbol{u} }\right),
\end{equation}
with $f(\boldsymbol{\theta }, \boldsymbol{u})$ denoting the empirical loss function corresponding to the data sample $\boldsymbol{u}$ in the local dataset $B_m$.

In FL with EH devices, unlike traditional FL, limited energy availability can leave some users without sufficient energy to perform local computations or transmit. Consequently, contributions will come only from MDs with sufficient energy and will be selected in accordance with the adopted scheduling policy. At each global iteration $t$, the PS broadcasts the latest global model, $\boldsymbol{\theta }(t)$. In response, the selected MDs perform $\tau$ local stochastic gradient descent (SGD) iterations to minimize their individual local loss functions, $F_m(\boldsymbol{\theta })$, for $m \in \mathcal{S}(t)$, where $\mathcal{S}(t)$ is the set of scheduled users in the $t$-th global iteration. Subsequently, the model updates obtained by mobile users are transmitted back to the PS, contributing to the global learning process.

To compute the local model updates, the $m$-th user (for the $i$-th local and $t$-th global iteration) employs the following update rule: 
\begin{equation} \label{lc_update_rule}
    \boldsymbol{\theta}_{m}^{i+1}(t)=\boldsymbol{\theta}_{m}^{i}(t)-\eta_{m}^{i}(t)\nabla F_{m}(\boldsymbol{\theta}_{m}^{i}(t), \xi_{m}^{i}(t)),
\end{equation}
where $i \in [\tau]$, $\eta_{m}^{i}(t)$ is the learning rate and $\nabla F_{m}\left(\boldsymbol{\theta}_{m}^{i}(t), \xi_{m}^{i}(t)\right)$ represents the stochastic gradient estimate for the $\boldsymbol{\theta}_{m}^{i}(t)$ and the local mini-batch sample $\xi_{m}^{i}$ randomly chosen from the local dataset $\mathcal{B}_{m}$.

After the local SGD steps, the $m$-th user computes the model update, which is aimed to be shared with the PS as 
\begin{equation} \label{lc_update}
    \Delta\boldsymbol{\theta}_{m}(t)=\boldsymbol{\theta}_{m}^{\tau}(t)-\boldsymbol{\theta}_{m}^{1}(t).
\end{equation}

Using OTA transmission over a fading MAC, the received signal at the $k$-th antenna of the PS at the $t$-th global iteration is given as
\begin{equation} \label{rec_signal_0}
    \boldsymbol{y}_{PS,k}(t)=\sum_{m\in \mathcal{S}(t)}\boldsymbol{h}_{m,k}(t)\circ \boldsymbol{x}_{m}(t)+\boldsymbol{z}_{PS,k}(t),  
\end{equation}
where $\boldsymbol{x}_{m}(t)$ is the signal transmitted by the $m$-th user with dimension $N$, and $\boldsymbol{h}_{m,k}(t) \in \mathbb{C}^{N}$ is the independent and identically distributed (i.i.d.) channel gains from the $m$-th user to the $k$-th antenna with complex Gaussian entries ${h}_{m,k}^{n}(t) \sim \mathcal{C N}(0, \sigma_{h}^{2})$.  Similarly, ${z}_{PS,k}^{n}(t)$ denotes the $n$-th entry of the channel noise, $\boldsymbol{z}_{PS,k}(t)$, which is i.i.d. circularly symmetric white Gaussian noise (AWGN), i.e., it is distributed according to $\mathcal{C N}(0, \sigma_{z}^{2})$.

The PS uses the received signals from the $K$ antennas to update the global model as
\begin{equation} \label{gl_update_part}
    \boldsymbol{\theta}_{PS}(t+1)=\boldsymbol{\theta}_{PS}(t)+ \Delta \hat{\boldsymbol{\theta}}_{PS}(t),
\end{equation}
where $\boldsymbol{\theta }_{\text{PS}}(t)$ represents the global model vector at global iteration $t$ and $\Delta \hat{\boldsymbol{\theta}}_{PS}(t)$ is the estimate of the average of the local updates. Note that if there were no noise or fading, the average of the local updates would be
\begin{equation} \label{gl_update_error_free}
    \Delta \boldsymbol{\theta}_{PS}(t)=\frac{1}{\left| \mathcal{S}(t) \right| }\sum_{m\in \mathcal{S}(t)}\Delta\theta_{m}(t). 
\end{equation}

\begin{figure*}[t]
    \centering
        \centering
        \includegraphics[trim={2.5cm 2.5cm 2.5cm 5.3cm},clip,width=1\textwidth]{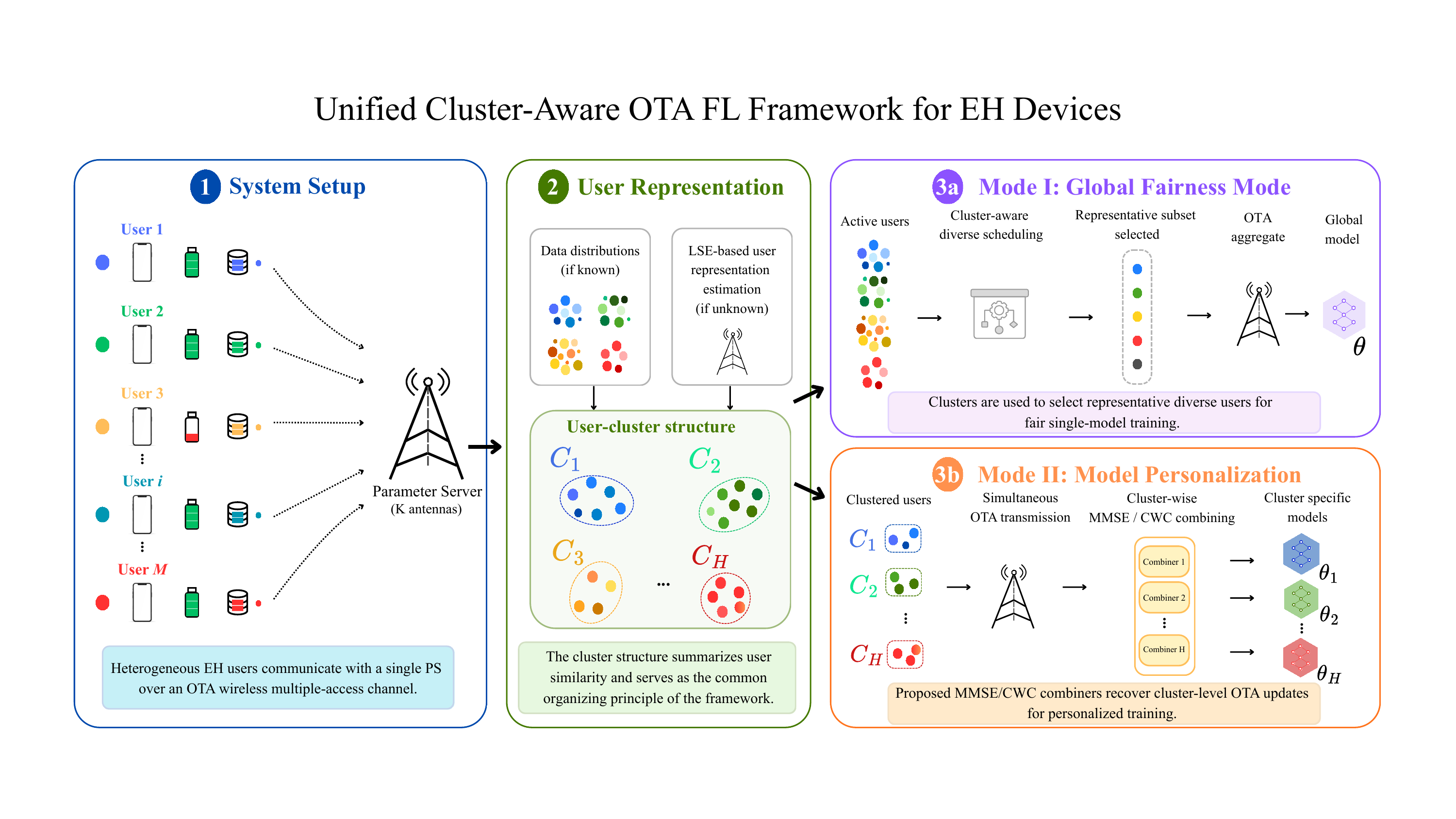}
        \vspace{-0.7cm}
        % \captionsetup{font=scriptsize}
\caption{Unified cluster-aware OTA-FL framework for EH devices. 
Heterogeneous EH users communicate with a single PS over an OTA wireless multiple-access channel, and their data/update similarities define a user-cluster structure. 
The framework uses this structure in two operating modes: Mode I performs cluster-aware diverse scheduling among active users to train a fair global model, while Mode II trains cluster-specific personalized models by recovering cluster-level OTA updates using the proposed MMSE/CWC combiners. }
        \label{fig:overview}
\end{figure*}
% \vspace{-0.7cm}

Building on the described general FL setup, we consider more practical scenarios in which the MDs experience highly heterogeneous data distributions, have limited energy due to energy harvesting, and operate under the no-CSIT assumption.
We consider a unified cluster-aware OTA-FL framework for EH devices, as illustrated in Fig.~\ref{fig:overview}, in which heterogeneous EH users communicate with a single PS over an OTA wireless multiple-access channel and participate intermittently based on their available energy. 
Their heterogeneous local data induce a natural user-cluster structure, which serves as the common basis for both operating modes of the proposed framework.
The proposed framework uses this cluster structure in two complementary operating modes. 
In Mode I, the global fairness mode, clusters guide diverse scheduling among active users to select a representative subset for fair single-model training. 
In Mode II, the model personalization mode, clusters define cluster-specific learning objectives and propose combiners that recover cluster-level OTA updates from simultaneous transmissions, thereby training personalized models at a single PS.

\subsection{\MakeUppercase{EH Devices}}

We study FL with OTA for EH devices with unit-sized batteries. At each global iteration, devices harvest energy with varying success and store it for future use, following the harvest-store-use approach \cite{EHsensors}. Surplus energy is wasted if the battery is full when energy arrives. Local SGD and update transmissions consume one energy unit per global iteration, highlighting stochastic energy availability as a key constraint. In this work, the EH constraint is modeled through device availability and participation, i.e., whether a user has sufficient harvested energy to perform local computation and update transmission in a given iteration, while transmit-power adaptation is beyond the scope of the considered model.

We consider a Bernoulli energy arrival process where, at each global iteration $t$, the $m$-th user receives unit energy with probability $p_{e}^{m}(t)$. Active users are those who have sufficient energy to participate in a given iteration. In the presence of scheduling, policies select clients from the pool of active users, whereas those who are active but unscheduled store their energy for future iterations. Energy arrivals are shared with the PS after each iteration, as also adopted in \cite{an2024online, chen2023joint}.

\subsection{\MakeUppercase{OTA FL with EH Devices with Unit Battery}}

We next consider FL with OTA aggregation, where only active users contribute to the iterations due to limited energy arrivals.
In this case, the model updates of the scheduled users, $\Delta \boldsymbol{\theta}_{m}^{cx}(t) \in \mathbb{C}^{N}$, $m\in \mathcal{S}(t)$, are transmitted as complex signals at iteration $t$, represented as 
\begin{subequations}
\begin{align} 
    \Delta \boldsymbol{\theta}_{m}^{r e}(t) &\triangleq \left[\Delta \theta_{m}^{1}(t), \Delta \theta_{m}^{2}(t), \ldots, \Delta \theta_{m}^{N}(t)\right]^{T} \label{complex_update_re}, \\
    \Delta \boldsymbol{\theta}_{m}^{i m}(t) &\triangleq \left[\Delta \theta_{m}^{N+1}(t), \Delta \theta_{m}^{N+2}(t) , \ldots, \Delta \theta_{m}^{2N}(t)\right]^{T} \label{complex_update_im}, \\
    \Delta \boldsymbol{\theta}_{m}^{cx}(t) &\triangleq \Delta \boldsymbol{\theta}_{m}^{r e}(t) + j \Delta \boldsymbol{\theta}_{m}^{i m}(t)
    \label{complex_update}.
\end{align}
\end{subequations}

Using the channel output at each antenna as given in \eqref{rec_signal_0}, the PS combines the signals from the $K$ antennas using the sum of the channel gains from the scheduled users as follows:
\begin{equation} \label{combined_signal}
    \boldsymbol{y}_{PS}(t)= \frac{1}{K}\sum_{k=1}^{K}\left(\sum_{m\in \mathcal{S}(t)} \boldsymbol{h}_{m,k}(t)\right)^{*}\circ \boldsymbol{y}_{PS,k}(t),
\end{equation}
where the received signal at the $k$-th antenna of the PS, $\boldsymbol{y}_{PS,k}(t)$, corresponding to the transmitted signal $\Delta\boldsymbol{\theta}_{m}^{cx}(t)$'s, is given in (\ref{rec_signal_0}).

The $n$-th symbol of (\ref{combined_signal}) can be partitioned into three parts
\begingroup
\allowdisplaybreaks
\begin{align} 
    y_{PS}^{n}(t)=&\underset{y_{PS}^{n,sig}(t)(\text{signal term})}{\underbrace{\sum_{m\in\mathcal{S}(t)}\left(\frac{1}{K}\sum_{k=1}^{K}\vert h_{m,k}^{n}(t)\vert ^{2}\right)\Delta\theta_{m}^{n,cx}(t)}} \notag\\ & \hspace{-0.9cm}+\underset{y_{PS}^{n,int}(t)(\text{interference term})}{\underbrace{\frac{1}{K}\sum_{m\in\mathcal{S}(t)}\underset{m^{\prime}\neq m}{\sum_{m^{\prime}\in\mathcal{S}(t)}}\sum_{k=1}^{K}(h_{m,k}^{n}(t))^{*}h_{m^{\prime},k}^{n}(t)\Delta\theta_{m^{\prime}}^{n,cx}(t)}} \notag \\ &\hspace{-0.9cm}+\underset{y_{PS}^{n,noise}(t)(\text{noise term})}{\underbrace{\frac{1}{K}\sum_{m\in\mathcal{S}(t)}\sum_{k=1}^{K}(h_{m,k}^{n}(t))^{*}z^{n}_{PS,k}(t)}} \label{rec_signal}.
\end{align}
\endgroup

As shown in \cite{amiri2020BFL}, the variance of the interference coefficient $\Delta \theta_{m^{\prime}}^{n, cx}(t)$ decreases with the number of antennas $K$. Hence, a sufficient number of antennas allows for accurate estimation and recovery of noisy aggregated updates as follows
\begin{subequations} \label{global_update}
\begin{align}
    \Delta\hat{\boldsymbol{\theta}}_{PS}^{n}(t)&=\frac{1}{ \left| \mathcal{S}(t) \right| \sigma_{h}^{2}}{\text{Re}}\{y_{PS}^{n}(t)\}, \\
    \Delta\hat{\boldsymbol{\theta}}_{PS}^{n+N}(t)&=\frac{1}{ \left| \mathcal{S}(t) \right| \sigma_{h}^{2}}{\text{Im}}\{y_{PS}^{n}(t)\},
\end{align}
\end{subequations}
with $n \in [N] $, to update the global model, as given in (\ref{gl_update_part}).

% ------------------------------------------------------------
% ------------------------------------------------------------
% ------------------------------------------------------------

% ------------------------------------------------------------
% ------------------------------------------------------------
% ------------------------------------------------------------

\section{\MakeUppercase{Update Estimation and Scheduling for OTA FL with EH Devices for Fair Global Model Training}}
\label{sec3}

In this section, we demonstrate that in setups that aim to train an unbiased single global model under highly heterogeneous conditions, users’ data distributions play a critical role in the scheduling procedure. This is especially important for highly non-i.i.d. data, where user scheduling can help minimize the error bound due to partial participation by EH devices. We first perform a convergence analysis for the OTA FL system with EH mobile devices. The results of this analysis are then utilized in our proposed approaches to minimize the error bound. Next, we propose an entropy-based user scheduling policy for known data distributions. We then extend our discussion to the case of unknown data distributions and show that the characteristics of the data distribution can be estimated via least-squares on user representations for scheduling.

% \vspace{-1cm}
\subsection{\MakeUppercase{Convergence Analysis}}
 
In this section, we provide a convergence analysis for the OTA FL with EH devices and no CSIT by upper-bounding the distance between our model estimate and the optimal model.

The optimal solution minimizing (\ref{gl_loss}) is $\boldsymbol{\theta}^{*} \triangleq \arg \min_{\boldsymbol{\theta}} F(\boldsymbol{\theta})$, with optimal loss $F^{*} = F\left(\boldsymbol{\theta}^{*}\right)$. For user $m \in [M]$, the optimal local model is $\boldsymbol{\theta}_{m}^{*} \triangleq \arg \min_{\boldsymbol{\theta}_{m}} F_{m}(\boldsymbol{\theta}_{m})$, with corresponding loss $F_{m}^{*} = F\left(\boldsymbol{\theta}_{m}^{*}\right)$.

% \subsubsection{Preliminaries}
\noindent\textit{1) Preliminaries:}
The amount of bias and heterogeneity across devices is represented by the following parameter
$\Gamma = F^{*} - \sum_{m=1}^{M} \frac{\left| \mathcal{B}_{m} \right|}{B} F^{*}_{m}.$
A high magnitude of $\Gamma$ indicates a significant non-i.i.d. data distribution, while $\Gamma \to 0$ reflects near i.i.d. data.

We consider the same learning rate across users and local iterations, \(\eta_m^i(t) = \eta(t)\), but allow it to vary between global iterations. The local model update at the $m$-th user for global iteration $t$ and local iteration $i \in [\tau]$ is given as 
\begin{equation} \label{updatess}
    \boldsymbol {\theta }_{m}^{i+1} (t)\,\,- \boldsymbol {\theta }_{m}^{1} (t)\,\,= - \eta (t)~\sum \limits _{l=1}^{i} \nabla F_{m} \left ({\boldsymbol {\theta }_{m}^{l} (t), \xi _{m}^{l} (t)~}\right).
\end{equation}

To perform a convergence analysis, following existing studies \cite{amiri2020BFL, busra2023, ozan2024}, we assume that the loss functions \(F_1, \ldots, F_M\) are all \(L\)-smooth and \(\mu\)-strongly convex. We note that the strong convexity assumption is adopted mainly to obtain a tractable theoretical characterization of the proposed framework.  In particular, the analysis identifies how OTA aggregation error, partial user participation, and scheduling-induced approximation error contribute to the convergence bound. The numerical experiments with neural networks are intended to empirically assess whether the resulting design insights generalize to non-convex learning models.

Also, it is assumed that the expected squared \(\ell_2\)-norm of the stochastic gradients is bounded; that is, for all \(i \in [\tau]\), \(m \in [M]\), and \(t\), we have $\mathbb{E}_{\xi} \left[ \left\| \nabla F_m \left( \boldsymbol{\theta}_m^i(t), \xi_m^i(t) \right) \right\|_2^2 \right] \leq G^2$.

% \subsubsection{Convergence Rate}
\noindent\textit{2) Convergence Rate:}
For the convergence analysis of OTA FL with EH devices and data heterogeneity, we compare the error-free system model with ours, where only a subset of users participate in each iteration. Using these findings, we will make user-scheduling decisions to minimize discrepancies in the resulting bound. Our main result is as follows.
\begin{theorem}  \label{thm1} For \( 0 < \eta(t) \leq \min \left\{ 1, \frac{1}{\mu \tau} \right\}, \forall t \). We have
\begin{align}\label{conv1}&\hspace {-.5pc}\mathbb {E} \left [{{ \left \|{{ \boldsymbol {\theta } (t)- {\boldsymbol {\theta }}^{*} }}\right \|_{2}^{2} }}\right ] \leq \left({{\prod \limits _{i=0}^{t-1} A(i) }}\right) \left \|{{ {\boldsymbol {\theta }} (0) - {\boldsymbol {\theta }}^{*} }}\right \|_{2}^{2} \notag\\& \qquad\qquad\qquad\qquad\qquad {+\, \sum \limits _{j=0}^{t-1} B(j) \prod \limits _{i=j+1}^{t-1} A(i),}\end{align}
with

\begingroup
\allowdisplaybreaks

\begin{align} \label{conv2} A(i)\triangleq&1 - \mu \eta (i)~\left ({\tau - \eta (i) (\tau - 1) }\right),\\ B(i)\triangleq& \frac { \eta ^{2}(i) \tau ^{2} G^{2}}{K} + \frac {\sigma _{z}^{2}N}{ K {\left| \mathcal{S}(i) \right|} \sigma _{h}^{2}} \notag\\&+ \left ({1+ \mu (1- \eta (i)) }\right) \eta ^{2}(i) G^{2} \frac {\tau (\tau -1)(2\tau -1)}{6} \notag\\&+ \eta ^{2}(i) (\tau ^{2} + \tau -1)  G^{2} + 2 \eta (i) (\tau - 1) \Gamma \notag\\
& +\left( \eta^2(i) \tau(\tau-1)LG + \eta(i)\tau \epsilon  \right)^{2} \notag\\
&+ \left( \eta^2(i) \tau(\tau-1)LG + \eta(i)\tau \epsilon \right) c,\label{conv3} \end{align} 
\endgroup
for gradient approximation error $\epsilon$ and some constant $c \geq 0$.
\label{theorem1}
\end{theorem}

\begin{proof} %%%%%%%%%%%%%%%%%% Thm1 Proof
See Appendix \ref{appendixA}.
\end{proof}

We note that A$(i)$ represents the decay rate of the distance from the initial starting point to the optimal solution. In $B(i)$, the first two terms represent the transmission error due to the wireless fading MAC with blind transmitters, and the third and fourth terms are related to federated averaging. Additionally, we emphasize that the last two terms represent the error caused by partial user participation, similar to \cite{balakrishnan2022diverse}, with $\epsilon$ in (\ref{conv3}) being the gradient approximation error defined as follows 
\begin{align} \label{epsilon}
\epsilon \!\triangleq \!\left\| %\left( 
\frac{1}{M} \!\!\sum_{m=1}^{M} \!\nabla F_m(\theta_m(t) \!)
\!-\!
\frac{1}{\left| S(t) \right|} \!\!\sum_{m \in \mathcal{S}(t)}  \!\!\!\nabla F_m(\theta_m(t)\!)   %\right)
\right\|_2 \!.
\end{align} 

% ---------------------------------------------------------------
% Corollary 1
% ---------------------------------------------------------------
\begin{cor}[Convergence under Decaying Stepsize]
\label{cor1}
Under the conditions of Theorem~\ref{thm1}, assume $|\mathcal{S}(t)| \geq S_{\min} > 0$ for all $t$, and that there exists a uniform bound $\epsilon(t) \leq \bar{\epsilon}$ for all $t$. Let the stepsize be $\eta(t) = c_0/(t+t_0)$, where $t_0$ is chosen sufficiently large so that $\eta(0) \leq \min\!\left\{1,\frac{1}{\mu\tau}\right\}$. Define $q_0 \triangleq \tau - (\tau-1)\eta(0) > 0$, $\alpha \triangleq \mu q_0 c_0$, and $\kappa_\alpha \triangleq (1 + 1/t_0)^\alpha$, and require $\alpha > 1$, i.e., $c_0 > 1/(\mu q_0)$.

\noindent\textbf{(Finite $K$):} The bound in Theorem~\ref{thm1} satisfies
\begin{align}
    &\mathbb{E}\!\left[\|\boldsymbol{\theta}(t) - \boldsymbol{\theta}^*\|_2^2\right]
    \leq \left(\frac{t_0}{t+t_0}\right)^{\alpha}
    \|\boldsymbol{\theta}(0)-\boldsymbol{\theta}^*\|_2^2 \notag\\
    &\quad + \underbrace{\frac{\sigma_z^2 N}{K S_{\min} \sigma_h^2}}_{\triangleq\, B_0}
    \cdot\Phi(t)
    + \frac{\kappa_\alpha B_1}{\mu q_0}
    \left[1 - \left(\frac{t_0}{t+t_0}\right)^{\alpha}\right]
    + R_{2,K}(t),
    \label{cor1:finiteK}
\end{align}
where $\Phi(t) \triangleq \sum_{j=0}^{t-1}\prod_{i=j+1}^{t-1}A(i)$, $B_1$ and $B_2$ are defined in~\eqref{eq:B1}--\eqref{eq:B2}, and $R_{2,K}(t)$ is defined in~\eqref{eq:R2K}. Since $A(i) = 1 - \Theta(1/i)$, the accumulated factor satisfies $\Phi(t) = \Theta(t)$, so the $B_0$ term grows linearly and the derived bound does not establish convergence to a finite neighborhood for fixed finite $K$. Such a guarantee would require the OTA noise contribution to vanish or decay with $t$, for example through increasing $K$, increasing transmit power, or an appropriately scaled communication model.

\noindent\textbf{(Large $K$):} As $K \to \infty$, $B_0 \to 0$, $\Phi(t)$ contributes nothing, and $R_{2,K}(t) \to R_2(t)$ as defined in~\eqref{eq:R2}, so the bound reduces to
\begin{align}
    &\mathbb{E}\!\left[\|\boldsymbol{\theta}(t) - \boldsymbol{\theta}^*\|_2^2\right]
    \leq \left(\frac{t_0}{t+t_0}\right)^{\alpha}
    \|\boldsymbol{\theta}(0)-\boldsymbol{\theta}^*\|_2^2 \notag\\
    &\quad + \frac{\kappa_\alpha B_1}{\mu q_0}
    \left[1 - \left(\frac{t_0}{t+t_0}\right)^{\alpha}\right]
    + R_2(t),
    \label{cor1:largeK}
\end{align}
where $R_2(t) = O(1/t)$. As $t \to \infty$, both the initial error term and $R_2(t)$ vanish, and the model converges to a residual neighborhood satisfying:
\begin{align}
    \limsup_{t\to\infty}\, \mathbb{E}\!\left[\|\boldsymbol{\theta}(t)-\boldsymbol{\theta}^*\|_2^2\right]
    \leq \frac{\kappa_\alpha B_1}{\mu q_0},
    \label{cor1:limsup}
\end{align}
whose size is governed solely by data heterogeneity $\Gamma$ and gradient approximation error $\bar{\epsilon}$ through $B_1 = 2(\tau-1)\Gamma + \tau\bar{\epsilon} c$.
\end{cor}

\begin{proof}
See Appendix \ref{appendixB}.
\end{proof}

% ---------------------------------------------------------------
% Remark 1
% ---------------------------------------------------------------
\begin{remark}[Connection to Scheduling Design]
\label{rem:scheduling}
Corollary~\ref{cor1} directly motivates the scheduling policies proposed in the following subsections. In the large-$K$ regime, the asymptotic residual neighborhood is governed solely by $\kappa_\alpha B_1/(\mu q_0)$, where $B_1 = 2(\tau-1)\Gamma + \tau\bar{\epsilon} c$ depends on the gradient approximation error $\bar{\epsilon}$ linearly. Since $\bar{\epsilon}$ measures the discrepancy between the full-participation gradient and the scheduled-subset gradient, reducing it via careful user selection directly shrinks the limiting neighborhood. The $\eta^2$-order terms involving $B_2^\infty$, which contain a quadratic dependence on $\bar{\epsilon}$, contribute only to the transient behavior through $R_2(t) = O(1/t)$ and vanish asymptotically; they nonetheless affect the multiplicative constant of the finite-time error bound, so reducing $\bar{\epsilon}$ improves both the asymptotic neighborhood and the finite-time error bound. This provides a principled theoretical justification for the diversity-aware scheduling strategies developed in the following subsections: by selecting users whose combined update best approximates the full-participation average, the proposed entropy-based and least-squares estimation (LSE)-based policies are designed to reduce, or approximately control, $\bar{\epsilon}$, thereby pushing the residual neighborhood downward. Furthermore, the OTA noise term $B_0 = \sigma_z^2 N/(K S_{\min}\sigma_h^2)$ reveals that ensuring a larger minimum scheduled set size $S_{\min}$ suppresses the channel noise contribution, providing an additional argument for participation-maximizing scheduling under the energy-harvesting constraints of the system.
\end{remark}

% ---------------------------------------------------------------
% Corollary 2
% ---------------------------------------------------------------
\begin{cor}[Steady-State Error under Constant Stepsize]
\label{cor2}
Under the conditions of Theorem~\ref{thm1}, assume that $|\mathcal{S}(t)| \geq S_{\min} > 0$ and $\epsilon(t) \leq \bar{\epsilon}$ for all $t$, and let $\bar{B}$ denote the corresponding uniform upper bound on $B(i)$ obtained by evaluating~\eqref{conv3} with $|\mathcal{S}(i)| = S_{\min}$ and $\epsilon = \bar{\epsilon}$. Let $\eta(t) = \eta$ be constant across all iterations, chosen such that $0 < \eta \leq \min\!\left\{1, \frac{1}{\mu\tau}\right\}$, so that $A \triangleq 1 - \mu\eta(\tau-\eta(\tau-1)) \in [0,1)$ is constant. Then the bound in Theorem~\ref{thm1} satisfies:
\begin{align}
    \mathbb{E}\!\left[\|\boldsymbol{\theta}(t) - \boldsymbol{\theta}^*\|_2^2\right]
    \leq A^t \|\boldsymbol{\theta}(0) - \boldsymbol{\theta}^*\|_2^2
    + \frac{\bar{B}}{1-A},
    \label{cor2:bound}
\end{align}
\begin{figure*}[t]
\normalsize
\begin{align}
    \frac{\bar{B}}{1-A} =\;
    &\underbrace{\frac{\sigma_z^2 N}{K S_{\min} \sigma_h^2\,\mu\eta(\tau-\eta(\tau-1))}}_{\text{OTA channel noise}}
    + \underbrace{\frac{\eta\tau^2 G^2/K + (1+\mu)G^2\eta\tau(\tau-1)(2\tau-1)/6 + \eta(\tau^2+\tau-1)G^2}{\mu(\tau-\eta(\tau-1))}}_{\text{stochastic gradient noise}}
    \notag\\
    &+ \underbrace{\frac{2(\tau-1)\Gamma + \tau\bar{\epsilon} c}{\mu(\tau-\eta(\tau-1))}}_{\text{data heterogeneity and participation bias}}
    + \underbrace{\frac{\eta\left[\tau(\tau-1)LG+\tau\bar{\epsilon}\right]^2 + \eta\,\tau(\tau-1)LGc}{\mu(\tau-\eta(\tau-1))}}_{\text{partial participation}}
    \label{cor2:decomp}
\end{align}
\hrulefill
\vspace*{4pt}
\end{figure*}
where the first term vanishes as $t \to \infty$, and the model converges to a steady-state error floor:
\begin{align}
    \limsup_{t\to\infty}\,
    \mathbb{E}\!\left[\|\boldsymbol{\theta}(t) - \boldsymbol{\theta}^*\|_2^2\right]
    \leq \frac{\bar{B}}{1-A}.
    \label{cor2:floor}
\end{align}
The floor $\bar{B}/(1-A)$ decomposes as in~\eqref{cor2:decomp}, using $1-A = \mu\eta(\tau-\eta(\tau-1))$.
\end{cor}

\begin{proof}
Under the stated assumptions, $A(i) = A \in [0,1)$ and $B(i) \leq \bar{B}$ for all $i$, so the bound in~\eqref{conv1} gives:
\begin{align}
    &\mathbb{E}\!\left[\|\boldsymbol{\theta}(t)-\boldsymbol{\theta}^*\|_2^2\right]
    \leq A^t\|\boldsymbol{\theta}(0)-\boldsymbol{\theta}^*\|_2^2
    + \bar{B}\sum_{j=0}^{t-1}A^{t-1-j} \notag\\
    &= A^t\|\boldsymbol{\theta}(0)-\boldsymbol{\theta}^*\|_2^2
    + \bar{B}\cdot\frac{1-A^t}{1-A}
    \leq A^t\|\boldsymbol{\theta}(0)-\boldsymbol{\theta}^*\|_2^2
    + \frac{\bar{B}}{1-A},
\end{align}
where the geometric sum is evaluated in closed form since $A < 1$. Taking $t \to \infty$, $A^t \to 0$, giving~\eqref{cor2:floor}. The decomposition~\eqref{cor2:decomp} follows from substituting $\bar{B}$ evaluated at constant $\eta$, $S_{\min}$, and $\bar{\epsilon}$ together with $1-A = \mu\eta(\tau-\eta(\tau-1))$ into $\bar{B}/(1-A)$, and grouping terms according to their respective sources of error.
\end{proof}

% ---------------------------------------------------------------
% Remark 2
% ---------------------------------------------------------------
\begin{remark}[Interpretation of the Steady-State Floor]
\label{rem:floor}
The decomposition in~\eqref{cor2:decomp} reveals the contribution of each system component to the irreducible error floor, and has several important implications.

First, the \emph{OTA channel noise} term scales as $\Theta(1/\eta)$ and therefore grows without bound as $\eta \to 0$, representing the dominant trade-off in the constant-stepsize regime. It also decreases with $K$ and $S_{\min}$, confirming that larger antenna arrays and higher minimum user participation suppress the wireless channel's impact on learning performance. In the limit $K \to \infty$, this term vanishes.

Second, the \emph{data heterogeneity and participation bias} term contains two contributions: $2(\tau-1)\Gamma/(\mu(\tau-\eta(\tau-1)))$, which approaches the finite constant $2(\tau-1)\Gamma/(\mu\tau)$ as $\eta \to 0$ and represents a fundamental limit imposed by the non-i.i.d. data distribution that cannot be eliminated by stepsize tuning; and $\tau\bar{\epsilon} c/(\mu(\tau-\eta(\tau-1)))$, which approaches the finite constant $\bar{\epsilon} c/\mu$ as $\eta \to 0$ and can be reduced through careful user scheduling.

Third, and most directly relevant to the proposed framework, the \emph{partial participation} term is $O(\eta)$ for fixed $L$, $G$, $\tau$, and $\bar{\epsilon}$, and vanishes as $\eta \to 0$. The component quadratic in $\bar{\epsilon}$ scales as $O(\eta\tau\bar{\epsilon}^2/\mu)$, while the entire term is $O(\eta)$. Together with the linear contribution $\bar{\epsilon} c/\mu$ in the previous term, $\bar{\epsilon}$ appears in multiple components of the floor, and reducing it through careful user selection reduces both. The diversity-aware scheduling strategies in the following subsections are designed to reduce, or approximately control, $\bar{\epsilon}$, thereby pushing the steady-state floor downward.

Finally, there is an inherent \emph{trade-off in the stepsize $\eta$}: the OTA channel noise term scales as $\Theta(1/\eta)$ and dominates as $\eta \to 0$; the stochastic gradient noise and partial participation terms are $O(\eta)$ and decrease with $\eta$; and the heterogeneity and linear-$\bar{\epsilon}$ components approach finite constants as $\eta \to 0$. The genuine trade-off is therefore between the OTA noise floor, which worsens with smaller $\eta$, and the gradient noise and quadratic participation terms, which improve. The optimal constant stepsize balances these competing effects and can be chosen, in principle, as a function of $\mu$, $\tau$, $\Gamma$, $\bar{\epsilon}$, $K$, and $S_{\min}$.
\end{remark}

\subsection{\MakeUppercase{Entropy-based User Scheduling with Known Data Distributions}}

Assuming that all MDs disclose their data distributions to the PS in advance, we can select a subset of users that effectively represent all data labels in the network. Based on this, the PS characterizes the label distribution of each user $m \in [M]$ as $L_{m} = \left[ l_{m,0}, l_{m,1}, \dots, l_{m,{N_c}-1} \right]$, where $N_c$ is the total number of classes, and $l_{m,{n_c}}$ represents the portion of the $m$-th user's data corresponding to label $n_c$. At each iteration, the PS computes the label distribution for all available user subsets as a probability mass function and selects the one with the highest Shannon entropy, indicating the most balanced distribution. While this strategy is similar to that in \cite{lutz2024entropybased}, we extend our approach to a more practical setup that incorporates OTA transmission, a wireless fading MAC, and blind transmitters, showing the effectiveness of entropy-based user selection for EH devices under practical constraints.

\subsection{\MakeUppercase{User Clustering and Scheduling with Unknown Data Distribution}}

We consider a more realistic scenario in which the PS does not know the user data distributions, thereby preserving user privacy. In this case, we rely on the relationship between users' model updates and the underlying data distribution, similar to \cite{wang2020rl, fraboni2021clusteredsampling, balakrishnan2022diverse}. Unlike these studies, our approach, due to OTA transmission, is constrained to using a noisy estimate of the sum of updates from all selected users at each iteration. 
We demonstrate that a representation of user updates can be estimated at the PS, enabling clustering based on similarities in these representations, thereby minimizing the error due to partial participation in \eqref{epsilon}. This approach selects suitable users while preventing redundant information transfer and conserving energy, subject to the constraints of EH devices.

To achieve this, we use LSE to construct a representation of the updates as follows.
Over $T$ estimation iterations, the PS stores normalized global updates from (\ref{global_update}) while all active users participate without scheduling, a phase termed the \textit{estimation phase}. Note that we normalize the received global updates to mitigate possible scale discrepancies. At the end of this estimation window, PS estimates the representative updates based on stored global updates and participation information. We emphasize that the goal is to estimate a representation of user updates rather than recovering the individual updates themselves.

We define a matrix $\hat{\boldsymbol{\Theta}}_{PS}$, whose rows represent global model updates $\Delta\hat{\boldsymbol{\theta}}_{PS}(t)$ from (\ref{global_update}): 
\begin{align}
\hat{\boldsymbol{\Theta}}_{PS} &= [\Delta \hat{\boldsymbol{\theta}}_{PS}(t-T+1); \cdots ;\Delta \hat{\boldsymbol{\theta}}_{PS}(t)] \in \mathbb{R}^{T \times 2N}\\
&= \begin{bmatrix}
% \hat{\boldsymbol{\Theta}}_{PS, t-T+1} \quad 
\hat{\boldsymbol{\Theta}}_{PS, t-T+1}, 
\dots, 
\hat{\boldsymbol{\Theta}}_{PS, j}, 
\dots,
\hat{\boldsymbol{\Theta}}_{PS, t} 
\end{bmatrix}^T_{T \times 2N}.
\end{align}

For the $j$-th iteration with $j\le T$,  the $j$-th row of this matrix can be expressed as 
\begin{align}
\hat{\boldsymbol{\Theta}}_{PS, j} = \boldsymbol{A}_{j} \boldsymbol{\Theta}_{j} + \boldsymbol{N}_{j}^{'},
\end{align}
where $\boldsymbol{A}_{j}$ is a binary participation vector with  $\boldsymbol{A}_j \in \{0,1\}^{1 \times M}$, and $\boldsymbol{\Theta}_{j} \in \mathbb{R}^{M \times 2N}$, with each row representing the local model update for a specific user $m \in [M]$, denoted as $\Delta \boldsymbol{\theta}_{j,m}$.
% different users. 
Additionally, $\boldsymbol{N}_{j}^{'} \in \mathbb{R}^{1 \times 2N}$, whose $d$-th element is denoted by ${N}_{j,d}^{'}$ for $d \in [2N]$, represents the effective noise arising from MAC fading, AWGN, and PS combining errors. 
We also define $\boldsymbol{\Theta}_{rep} \in \mathbb{R}^{M \times 2N}$ as a representation of local updates. Using this, $\hat{\boldsymbol{\Theta}}_{PS, j}$ can be written as:
\begin{align} \label{est1}
 \hat{\boldsymbol{\Theta}}_{PS, j} =& \boldsymbol{A}_{j} (\boldsymbol{\Theta}_{rep} + \boldsymbol{\Theta}_{\textit{diff},j}) + \boldsymbol{N}_{j}^{'},
\end{align}
where $\boldsymbol{\Theta}_{\textit{diff},j}$ is defined as the difference between $ \boldsymbol{\Theta}_{j} - \boldsymbol{\Theta}_{rep} $.
Combining (\ref{est1}) for $j \in \{1, \dots, T\}$ and defining a total noise term $\boldsymbol{N}^{*}_{j} \triangleq \boldsymbol{A}_{j} \boldsymbol{\Theta}_{\textit{diff},j} + \boldsymbol{N}^{'}_{j}$, which represents the noise due to the channel, interference from the blind transmitters, and the difference between representative updates and the real updates, we obtain
\begin{equation} \label{rep_grad}
\hat{\boldsymbol{\Theta}}_{PS} = \boldsymbol{A} \boldsymbol{\Theta}_{rep} + \boldsymbol{N}^{*},
\end{equation}
where $\boldsymbol{A} \in \{0,1\}^{T \times M}$, $\boldsymbol{\Theta}_{rep} \in \mathbb{R}^{M \times 2N} $ and $\boldsymbol{N}^{*} \in \mathbb{R}^{T \times 2N}$.

By solving the LSE of $\boldsymbol{\Theta}_{rep}$ in (\ref{rep_grad}), we can get an estimate for the representative updates as $\hat{\boldsymbol{\Theta}}_{rep}$. Using this representation, the PS can infer characteristics of the users' data distribution by measuring similarity between user representations, which can then be used in the user selection procedure. Notably, the PS infers similarity among user representations without accessing their data distribution, preserving user privacy.

Due to the limited and stochastic nature of energy arrivals, some users may dominate the training and introduce bias toward specific labels and users.
By employing \textit{cosine similarity}, users are clustered to promote diverse user contributions, with the expected number of users per cluster determined by each cluster's energy distribution to ensure unbiased training. This approach helps reduce bias arising from non-i.i.d. data and provides fair performance across users, as noted in \cite{wang2020rl, balakrishnan2022diverse}.

\subsection{\MakeUppercase{Visualization of Cosine Similarity Based Clustering}}

% \begin{figure}
%     \centering
%     % First figure
%     \begin{subfigure}[t]{0.49\columnwidth}
%         \centering
%         \includegraphics[trim={1.9cm 0.45cm 1cm 0.83cm},clip,width=\linewidth]{figures/ch3/updates_cos45_new.pdf}
%         \vspace{-0.6cm}
%         \captionsetup{font=scriptsize}
%         \caption{Cosine similarity on real user updates.}
%         \label{fig:updates_cos45}
%     \end{subfigure}
%     % Second figure
%     \begin{subfigure}[t]{0.49\columnwidth}
%         \centering
%         \includegraphics[trim={1.9cm 0.45cm 1cm 0.83cm},clip,width=\linewidth]{figures/ch3/rep_cos45_new.pdf}
%         \vspace{-0.6cm}
%         \captionsetup{font=scriptsize}
%         \caption{Cosine similarity on estimated user representations.}
%         \label{fig:rep_cos45}
%     \end{subfigure}
%     \vspace{-0.15cm}
%     \captionsetup{font=footnotesize}
%     \caption{The visualization of cosine similarity on the user clusters.}
%     \label{fig:cos45}
%     % \vspace{-0.7cm}
% \end{figure}

\begin{figure}[t]
\centering

\subfloat[Cosine similarity on real user updates.]{
\includegraphics[trim={1.9cm 0.45cm 1cm 0.83cm},clip,width=0.44\columnwidth]{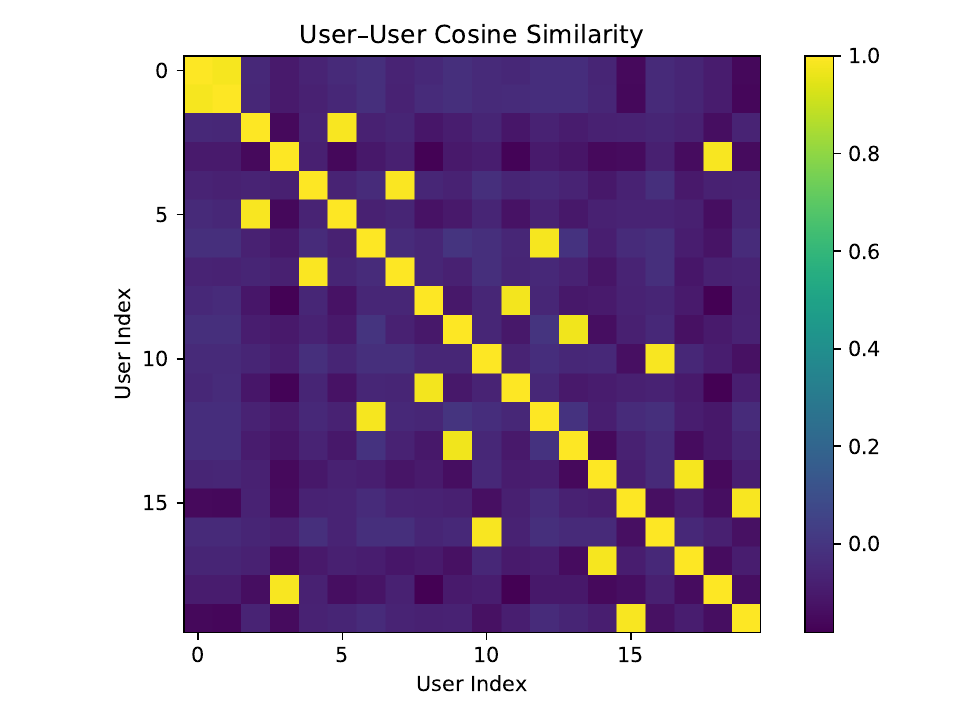}
\label{fig:updates_cos45}
}
\hspace{0.02\columnwidth}
\subfloat[Cosine similarity on estimated user representations.]{
\includegraphics[trim={1.9cm 0.45cm 1cm 0.83cm},clip,width=0.44\columnwidth]{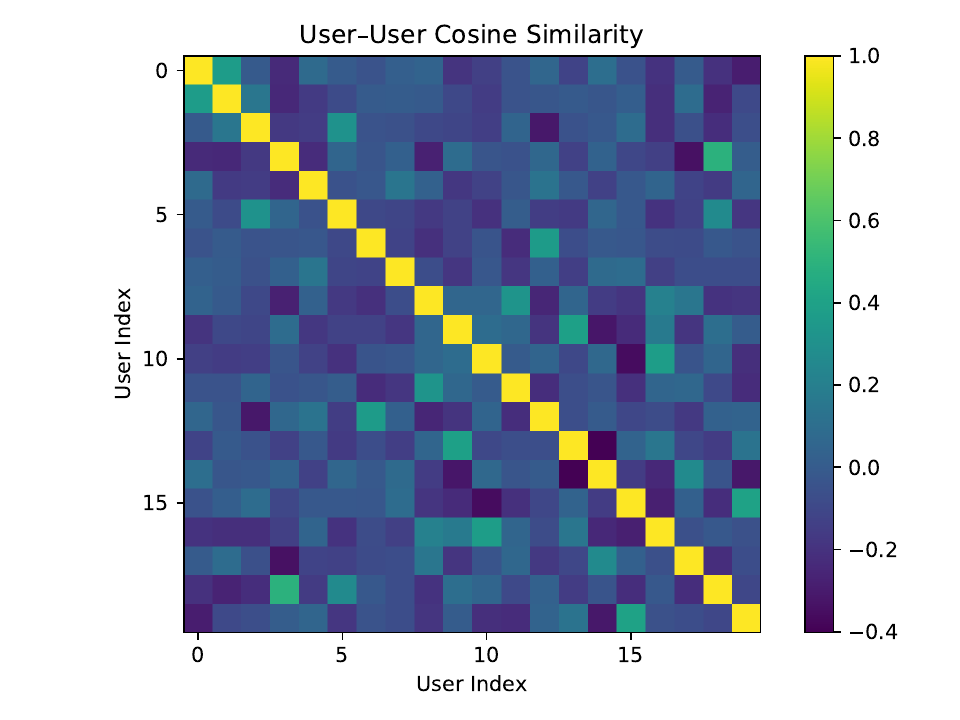}
\label{fig:rep_cos45}
}

\caption{The visualization of cosine similarity on the user clusters.}
\label{fig:cos45}
\end{figure}

To illustrate the relevance of cosine similarity in our setting, we consider a simple example using the MNIST dataset. A single-layer neural network with $2N = 7850$ parameters is trained using 20 MDs, each holding data from a single class. We compute the cosine similarity between users’ model updates at the initial iteration and compare it with the cosine similarity obtained from user representations estimated via the LSE method.

Fig.~\ref{fig:cos45} shows the cosine similarity matrices based on the true updates and their estimated representations. As observed in Fig.~\ref{fig:updates_cos45} and Fig.~\ref{fig:rep_cos45}, the estimated similarities closely match the true similarity patterns, indicating that the proposed approach accurately captures relationships among users. This confirms that user characteristics and the underlying clustering structure can be inferred directly from aggregated over-the-air signals.

% ------------------------------------------------------------
% ------------------------------------------------------------
% ------------------------------------------------------------

% ------------------------------------------------------------
% ------------------------------------------------------------
% ------------------------------------------------------------

\section{\MakeUppercase{Personalized Clustered Federated Learning with OTA Aggregation for EH Devices}}
\label{sec4}

While the user scheduling approaches for highly heterogeneous FL setups, proposed in the previous section, provide a way to train a fair global model and reduce training bias caused by data heterogeneity, many real-world applications (e.g., recommendation systems or online journals) can instead benefit from this heterogeneity by training multiple personalized models for users with specific preferences. However, training multiple personalized models for different clusters typically requires separate transmissions (e.g., sequential updates from all active users) from each cluster, or multiple parameter servers, which increases either latency or system cost. To address this, in our subsequent work, we propose an over-the-air personalized clustered FL approach that can independently recover the local updates of each cluster, even with synchronized transmission, using a single parameter server.

\subsection{\MakeUppercase{System Model for the Personalized CFL with EH Devices}}

We consider a CFL system with heterogeneous data distributions for $M$ EH devices. In this setup, the goal is to minimize a global loss function collaboratively across $H$ clusters, with each cluster $h \in [H]$ maintaining a personalized, distinct model. Each cluster $h$ is composed of a subset of the $M$ devices, where the clusters are mutually exclusive, that is, no device belongs to more than one cluster. For each cluster, the global model parameters $\boldsymbol{\theta}_h \in \mathbb{R}^{2N}$ are optimized. Each model $\boldsymbol{\theta}_h$ is designed to minimize the local loss function corresponding to its cluster $h$. The global loss function is defined as:
\begin{equation}
F(\boldsymbol{\theta}_1, \boldsymbol{\theta}_2, \dots, \boldsymbol{\theta}_H) =  \sum_{h=1}^{H} \sum_{m \in \mathcal{C}_h} \frac{|\mathcal{B}_m|}{B} F_{m}(\boldsymbol{\theta}_h),
\end{equation}
where $F_{m}(\boldsymbol{\theta}_h)$ represents the average empirical local loss for the $m$-th user in cluster $h$ with model parameters $\boldsymbol{\theta}_h$, and $\mathcal{C}_h$ is the set of users in cluster $h \in [H]$. For the $m$-th user with local dataset $\mathcal{B}_m$, the local loss function is defined as:
\begin{equation}
F_m(\boldsymbol{\theta}_h) = \frac{1}{|\mathcal{B}_m|} \sum_{\boldsymbol{u} \in \mathcal{B}_m} f(\boldsymbol{\theta}_h, \boldsymbol{u}),
\end{equation}
where $f(\boldsymbol{\theta}_h, \boldsymbol{u})$ is the empirical loss corresponding to the data sample $\boldsymbol{u}$ in the local dataset $\mathcal{B}_m$.

Consistent with the earlier discussion, EH devices are modeled using a Bernoulli energy-arrival process with harvesting probability $p_e^m(t)$, and each iteration consumes 1 unit of energy for local computation and transmission of updates. However, unlike the setup in the previous section, we assume that the harvested energy is immediately consumed in the next iteration. That is, no energy storage or scheduling is performed in this setting. We define $\mathcal{S}_h(t)$ as the set of active and participating users in cluster $h$ at global iteration $t$, with $\sum_{h=1}^{H} \mathcal{S}_h(t) = \mathcal{S}(t)$. Note that $\mathcal{S}_h(t)$ is a subset of the overall user set in the cluster, i.e., $\mathcal{S}_h(t) \subseteq \mathcal{C}_h$, where $\mathcal{C}_h$ denotes the users assigned to cluster $h$.

Our approach to clustering similar users based on their data distributions and update directions enables training a separate personalized model for each cluster, tailored to their observed data patterns. When users are using EH devices, limited energy availability may leave some users without sufficient energy to perform local computations and transmissions. In such cases, similar users within the group can make meaningful contributions and help one another develop a more personalized model for that cluster, rather than relying on the standard FL approach. In the standard FL framework, all the users are treated uniformly, and a single global model is learned, which often fails to serve all users effectively due to underlying data heterogeneity. Furthermore, unlike the user-scheduling approach proposed in the previous section, which targets a single global model that represents all users, personalized CFL enables specialized models tailored to user preferences, thereby capturing diverse user requirements more effectively.

At the $t$-th global iteration, the PS broadcasts the latest global models, $\boldsymbol{\theta}_1(t), \boldsymbol{\theta}_2(t), \dots, \boldsymbol{\theta}_H(t)$, corresponding to each cluster $h \in [H]$. In response, the active mobile devices in each cluster $h$ perform $\tau$ iterations of local SGD to minimize their individual local loss functions, $F_m(\boldsymbol{\theta})$, for $m \in \mathcal{S}_h(t)$. 
Subsequently, the model updates obtained by the MDs are transmitted back to the PS to contribute to the global learning process, with each cluster providing its local update to improve the cluster-level global models.

To compute local model updates in clustered FL, the $m$-th user in cluster $h$ performs local SGD at the $i$-th local and $t$-th global iteration according to \eqref{lc_update_rule}. After completing the local steps, the user computes the model update to be shared with the PS, as given in \eqref{lc_update}.
Once all users have computed their model updates locally, they simultaneously transmit them to the PS via over-the-air aggregation. Similar to the previous section, the model updates of the users, $\Delta \boldsymbol{\theta}_{m}^{cx}(t) \in \mathbb{C}^{N}$, $m\in \mathcal{S}(t)$, are transmitted as complex signals at iteration $t$, represented as \eqref{complex_update}. The PS combines the received signals from all users in each cluster, yielding cluster-specific global updates that are subsequently used to refine the cluster-specific global models.

In an ideal setup, where the server can perfectly identify the individual user updates, each model can be updated easily by averaging the updates within the corresponding cluster as:
\begin{equation} \label{gl_update_error_free_clustered}
    \Delta \boldsymbol{\theta}_h(t) = \frac{1}{|\mathcal{S}_h(t)|} \sum_{m \in \mathcal{S}_h(t)} \Delta\boldsymbol{\theta}_m(t), 
\end{equation}
where $h \in [H]$ represents each cluster, and $\mathcal{S}_h(t)$ is the set of participating users in cluster $h$ at global iteration~$t$, as mentioned earlier.

After averaging the updates within each cluster, the global model for cluster $h$ can be updated as:
\begin{equation} \label{gl_update_part_clustered}
    \boldsymbol{\theta}_h(t+1) = \boldsymbol{\theta}_h(t) + \Delta \boldsymbol{\theta}_h(t),
\end{equation}
where $\Delta \boldsymbol{\theta}_h(t)$ represents the aggregated update for the model of cluster $h$, based on the updates from all users in that cluster.

\subsection{\MakeUppercase{Personalized Clustered FL with OTA Aggregation}}
We assume that mobile devices lack CSI and transmit their updates to the PS over a fading MAC channel, employing OTA aggregation. The PS is equipped with $K$ receive antennas to align the received signals, even in the absence of CSI at the transmitters, by leveraging aggregated channel state information. The key challenge is that the server needs to recover the sum of the local updates for each cluster. However, when users send their updates concurrently over the wireless channel, the PS only observes the sum of the local updates from all users. Since the PS receives only the aggregated global over-the-air update, it cannot distinguish among cluster contributions, making it challenging to recover individual cluster updates.
This problem is exacerbated by the absence of CSI at the transmitters, leaving users blind to channel conditions and making it significantly more difficult for the PS to separate and correctly attribute aggregated updates to their respective clusters.

A potential solution is to leverage channel knowledge on the server side. Using the CSI that the PS can gather, one can employ minimum mean-square error (MMSE) estimation at the receiver to recover each user's individual updates. In this case, the PS can use the aggregated received signal and apply the MMSE technique to estimate the individual updates for each cluster, despite the concurrent transmissions. The MMSE approach separates the mixed signals of different clusters, enabling the server to accurately recover each cluster's local updates by exploiting signals observed at the different receive antennas. Using MMSE, the server can effectively perform CFL even without CSI at the users' side. This enables the PS to aggregate updates from each cluster, recover individual contributions, and perform necessary model updates while maintaining the benefits of concurrent transmission over a shared wireless channel.

Refer to \eqref{rec_signal_0}, which describes the received signal at the $k$-th antenna of the PS in the OTA-FL setup with EH devices.
Equivalently, in the personalized clustered FL setup, the received signal at the PS for the $k$-th antenna at iteration $t$ can be expressed as a sum over clusters, with each cluster $h$ transmitting its updates. The received signal becomes:
\begin{equation} \label{rec_signal_clustered}
    \boldsymbol{y}_{PS,k}(t) = \sum_{h \in [H]} \sum_{m \in \mathcal{S}_h(t)} \boldsymbol{h}_{m,k}(t) \circ \boldsymbol{x}_m(t) + \boldsymbol{z}_{PS,k}(t),
\end{equation}
where the first sum runs over the clusters, and the second sum runs over the users within each cluster, $\mathcal{S}_h(t)$, at global iteration $t$. 
Our goal is to design a combining technique at the receivers to combine $ \boldsymbol{y}_{PS,k}(t) $ for each cluster $ h \in [H] $ to ensure convergence guarantees for the personal model of each cluster instead of the single global model, and update the models as in  (\ref{gl_update_part_clustered}).

The core challenge lies in the accuracy of signal recovery with limited CSI, creating a trade-off between estimation accuracy and signaling overhead.
% in wireless federated learning systems. 
To address this, we investigate different estimation strategies based on the granularity of CSI available at the PS. Full per-user CSI can be obtained via pilot-based training, in which each user transmits known pilot symbols, enabling the PS to estimate each user's channel response. In multi-carrier systems such as OFDM, these pilots can be embedded on designated subcarriers, thereby avoiding the need to occupy the entire bandwidth~\cite{amiri2020dsgd, busra2023}. 
Assigning distinct pilot subcarriers to all users minimizes interference and enables accurate channel estimation; however, it incurs significant signaling overhead and becomes impractical with many users or rapidly varying channels.

To reduce overhead, the PS may instead rely on partial CSI, such as per-cluster channel information. In this setting, the sum of the channel coefficients for users within a cluster is obtained using common pilot subcarriers assigned to all users in that cluster, rather than separate pilots for each user. This allows the PS to recover the effective aggregated channel for each cluster, which we refer to as partial cluster-level CSI throughout this work.

% \subsubsection{\MakeUppercase{Full User-Level CSI Available at the PS}}
\noindent\textit{1) Full User-Level CSI Available at the PS:}

In this setting, the PS has full CSI for each user and each antenna at the PS. That is, it knows the complete channel coefficient array $\boldsymbol{H_f} \in \mathbb{C}^{N \times K \times M_s}$, where 
$M_s$ is the total number of active users and defined as $M_s \triangleq |\mathcal{S}(t)|$. With this detailed channel knowledge, the PS can apply estimation and combining methods to recover either per-user updates or aggregate cluster-level updates. This approach enables more accurate signal separation but requires extensive CSI estimation, which may be infeasible due to the associated cost and overhead in practical wireless FL systems.

% \paragraph{MMSE Combining with Full CSI} 

\indent\textit{a) MMSE Combining with Full CSI}

For the MMSE combining with the full CSI, we consider the equivalent channel model for (\ref{rec_signal_clustered}). Specifically, for the $n$-th symbol, the channel model is expressed as:
\begin{equation} \label{channel_model}
\boldsymbol{Y}_n = \boldsymbol{H}_n \boldsymbol{X}_n + \boldsymbol{N}_n,
\end{equation}
where $ \boldsymbol{H}_n \in \mathbb{C}^{ K \times M_s} $ is the channel matrix known at the receiver side, where each element of the matrix corresponds to the $h_{m,k}$ defined in (\ref{rec_signal_0}) and used in (\ref{rec_signal_clustered}). Also, $ \boldsymbol{Y}_n \in \mathbb{C}^{ K} $ is the received signal, $ \boldsymbol{X}_n \in \mathbb{C}^{ M_s} $ is the transmitted signal, and $ \boldsymbol{N}_n \in \mathbb{C}^{ K} $ is the additive noise. 
Note that for ease of illustration, we omit the iteration index $t$ from the parameters.

Given the received signal $\boldsymbol{Y}_n$ and the full channel matrix $\boldsymbol{H}_n$, PS aims to recover the transmitted signal $\boldsymbol{X}_n$ using a linear estimator of the form $\hat{\boldsymbol{X}_n} = \boldsymbol{W}_n \boldsymbol{Y}_n$. The MMSE estimator is obtained by minimizing the mean squared error (MSE) between the true and estimated signals, i.e., $\mathbb{E} \left[ \| \boldsymbol{X}_n - \boldsymbol{W}_n \boldsymbol{Y}_n \|_F^2 \right]$, yielding\footnote{The derivation follows standard linear MMSE estimation and is omitted for brevity.}
\begin{equation} \label{mmse_eq}
     \hat{\boldsymbol{X}_n} = \boldsymbol{C_{xx}} \boldsymbol{H}_n^H \left( \boldsymbol{H}_n \boldsymbol{C_{xx}} \boldsymbol{H}_n^H + \boldsymbol{C_{nn}} \right)^{-1} \boldsymbol{Y}_n .
\end{equation}
This is the MMSE estimator for $\boldsymbol{X}_n$, given the observation $\boldsymbol{Y}_n$, full CSI $\boldsymbol{H}_n$, and known signal and noise statistics.

Once the transmitted updates are estimated as $\hat{\boldsymbol{X}_n}$, they correspond to the recovered versions of individual user updates $\Delta \hat{\boldsymbol{\theta}}_m(t)$ for each scheduled user $m \in \mathcal{S}(t)$. The PS can then compute the aggregated update for each cluster by averaging the estimated updates of users assigned to cluster $h$ as:
\begin{equation} \label{gl_update_error_free_clustered_estimated}
    \Delta \hat{\boldsymbol{\theta}}_h(t) = \frac{1}{|\mathcal{S}_h(t)|} \sum_{m \in \mathcal{S}_h(t)} \Delta\hat{\boldsymbol{\theta}}_m(t),
\end{equation}
where $h \in [H]$ denotes the cluster index and $\mathcal{S}_h(t)$ is the corresponding set of scheduled users. These aggregated updates are then used to update the cluster-specific global models as: $    \boldsymbol{\theta}_h(t+1) = \boldsymbol{\theta}_h(t) + \Delta \hat{\boldsymbol{\theta}}_h(t),$
allowing each cluster model to evolve independently based on updates from its own user group.

% \subsubsection{\MakeUppercase{Partial Cluster-Level CSI Available at the PS}}
\noindent\textit{2) Partial Cluster-Level CSI Available at the PS:}

In this setting, the PS lacks access to individual user-level CSI. Instead, it obtains partial channel knowledge in the form of aggregated CSI at the cluster level. Specifically, for each cluster \( h \in [H] \), the PS is assumed to know the aggregated channel vector equal to the sum of the channel vectors of users in that cluster. For the \( k \)-th antenna, this can be written as \( \sum_{m \in \mathcal{S}_h(t)} \boldsymbol{h}_{m,k}(t) \).
These aggregated vectors are used to construct the cluster-level channel matrix \( \boldsymbol{H}_{c,n} \in \mathbb{C}^{K \times H} \) for the $n$-th symbol, where each column corresponds to a cluster. The estimation of such aggregated CSI can be performed by assigning a common pilot signal to all users within a cluster, allowing the PS to capture the superimposed channel response for that cluster in a single measurement. This approach significantly reduces CSI acquisition overhead relative to the full per-user CSI case, but it also limits the signal resolution and the ability to distinguish among users within the same cluster.

% \paragraph{MMSE Combining with Partial CSI}
\indent\textit{a) MMSE Combining with Partial CSI}

The $n$-th symbol of the received signal still follows the standard linear model introduced in \eqref{channel_model}: $\boldsymbol{Y}_n = \boldsymbol{H}_{c,n} \boldsymbol{X}_{c,n} + \boldsymbol{N}^{'}_n,$ where $\boldsymbol{Y}_n \in \mathbb{C}^{ K \times 1}$ denotes the received signal across $K$ antennas for symbol $n$,
$\boldsymbol{H}_{c,n} \in \mathbb{C}^{K \times H}$ is the aggregated cluster-level channel matrix,
$\boldsymbol{X}_{c,n} \in \mathbb{C}^{ H}$ is the matrix of transmitted cluster updates
(each column corresponding to one cluster),
and $\boldsymbol{N}^{'}_n \in \mathbb{C}^{K \times 1}$ denotes additive white Gaussian noise, 
which also captures intra-cluster interference arising from users experiencing different channels.

Each transmitted cluster signal $\boldsymbol{X}_{c,n}$ is assumed to be the sum of updates from all the users in the corresponding cluster. Since individual user-level recovery is not possible in this case, the PS applies a linear MMSE estimator to recover cluster-level updates. 
We also note that the estimated $\hat{\boldsymbol{X}}_{c,n} \in \mathbb{C}^{H }$ corresponds to the cluster-level aggregated updates. Each row of $\hat{\boldsymbol{x}}_{c,n}^{(h)}$ represents the MMSE estimate of the aggregated update from cluster $h$. 
We first define $\boldsymbol{A} \in \{0,1\}^{M \times H}$ as the cluster assignment matrix, mapping each of the $M$ users to one of the $H$ clusters, with orthogonal columns. We also define the cluster-sum projector as $\boldsymbol{P} \triangleq \boldsymbol{A}  (\boldsymbol{A}^{T}\boldsymbol{A})^{-1} \boldsymbol{A}^{T}$. 
This results in the cluster-summed local updates
$\boldsymbol{X}_{c,n} \triangleq \boldsymbol{A}^{T} \boldsymbol{X}_n$
and the associated clustered channel matrix
$\boldsymbol{H}_{c,n} \triangleq \boldsymbol{H}_n \boldsymbol{A}
(\boldsymbol{A}^{T}\boldsymbol{A})^{-1}$. Hence, the corresponding input-output relationship can be rewritten as
\begingroup
\allowdisplaybreaks
\begin{align}
& \boldsymbol{Y}_n = \boldsymbol{H}_n \boldsymbol{X}_n + \boldsymbol{N}^{'}_n \\
& \hspace{0.2cm}=  \boldsymbol{H}_n \boldsymbol{P} \boldsymbol{X}_n + 
\boldsymbol{H}_n (\boldsymbol{I}-\boldsymbol{P}) \boldsymbol{X}_n +\boldsymbol{N}^{'}_n \\
&\hspace{0.2cm}= \boldsymbol{H}_n \boldsymbol{A} ( \boldsymbol{A}^{T}\boldsymbol{A})^{-1} \boldsymbol{A}^{T} \boldsymbol{X}_n + 
\boldsymbol{H}_n (\boldsymbol{I}-\boldsymbol{P}) \boldsymbol{X}_n +\boldsymbol{N}^{'}_n \\
&\hspace{0.2cm}=
\boldsymbol{H}_{c,n} \boldsymbol{X}_{c,n} +  
\boldsymbol{H}_n (\boldsymbol{I}-\boldsymbol{P}) \boldsymbol{X}_n +\boldsymbol{N}^{'}_n\\
&\hspace{0.2cm} =
\boldsymbol{H}_{c,n} \boldsymbol{X}_{c,n} +\boldsymbol{\Tilde{N}}_n
\end{align}
\endgroup
where $\boldsymbol{\tilde{N}}_n$ represents the joint noise term, which consists of the effective noise $\boldsymbol{N}'_n$ and the interference term $\boldsymbol{H}_n(\boldsymbol{I}-\boldsymbol{P})\boldsymbol{X}_n$.
The latter captures the intra-cluster deviations of the local updates after clustering, and can be interpreted as a structured noise component arising from imperfect alignment among users within the same cluster.

Hence, estimating \( \boldsymbol{X}_{c,n} \) using only cluster-level partial CSI yields the average of the local updates across clusters. Then, the cluster-level MMSE estimate becomes 
\begin{equation} \label{cluster_MMSE}
    \hat{\boldsymbol{X}}_{c,n} \!= \boldsymbol{C}_{x_c x_c} \boldsymbol{H}_{c,n}^H \left( \boldsymbol{H}_{c,n} \boldsymbol{C}_{x_c x_c} \boldsymbol{H}_{c,n}^H \!+\! \boldsymbol{C}_{nn} \right)^{\!-1} \boldsymbol{Y}_n,
\end{equation} 
with $\boldsymbol{C}_{x_c x_c} = \boldsymbol{A}^{T} \boldsymbol{C}_{xx} \boldsymbol{A}$. Solving \eqref{cluster_MMSE} in a similar manner to \eqref{mmse_eq}, by substituting $\boldsymbol{C}_{xx} = \boldsymbol{C}_{x_c x_c}$ and $\boldsymbol{H}_n = \boldsymbol{H}_{c,n}$, and using the received signal from all antennas $\boldsymbol{Y}_n$ along with the cluster-level CSI $\boldsymbol{H}_{c,n}$, one can directly obtain the desired signal, i.e., the average of the transmitted updates from each cluster, without needing to recover individual user updates.
This shows that the cluster-level estimate is simply the sum of the user-level MMSE estimates within each cluster.

Using the estimate $\hat{\boldsymbol{X}}_{c,n}$, the PS can construct cluster-specific model updates in the same manner as \eqref{gl_update_error_free_clustered_estimated}.
This approach enables the PS to recover aggregated updates from each cluster, even with partial CSI availability, thereby enabling effective cluster-level model updates without requiring individual user-level CSI.

% \paragraph{Cluster-wise Weighted Combining (CWC) with Partial CSI}
\indent\textit{b) Cluster-wise Weighted Combining (CWC) with Partial CSI}

Using the received signal defined in \eqref{rec_signal_clustered}, the PS can isolate the contribution of cluster $h \in [H]$ by applying a channel-weighted combining operation. Specifically, to get the combined signal for cluster $h$, the PS performs:
\begin{equation} \label{cluster_combining_bfl}
\boldsymbol{y}_{PS}^{(h)}(t) = \frac{1}{K} \sum_{k=1}^{K} \left( \sum_{m \in \mathcal{S}_h(t)} \boldsymbol{h}_{m,k}(t) \right)^* \circ \boldsymbol{y}_{PS,k}(t).
\end{equation}
% where $\circ$ denotes element-wise (Hadamard) product. 
This operation acts as a beamformer targeting cluster $h$.

We focus on the $n$-th symbol of the PS's beamformed signal for cluster $h$, denoted by $y_{\mathrm{PS}}^{n,(h)}(t)$:
% Using the combining rule from (\ref{cluster_combining_bfl}), this symbol can be written as:
\begin{equation}
    y_{PS}^{n,(h)}(t) = \frac{1}{K} \sum_{k=1}^{K} \left( \sum_{m \in \mathcal{S}_h(t)} h_{m,k}^{n}(t) \right)^* \cdot y_{PS,k}^{n}(t).
\end{equation}
In this expression, $h_{m,k}^{n}(t)$ is the channel coefficient from user $m \in \mathcal{S}_h(t)$ to antenna $k$ for the $n$-th symbol, and $y_{PS,k}^{n}(t)$ denotes the received symbol at antenna $k$. The sum $\sum_{m \in \mathcal{S}_h(t)} h_{m,k}^{n}(t)$ forms a cluster-specific channel signature at antenna $k$.
% , and the conjugate $(\cdot)^*$ acts as a matched filter.
Averaging across antennas aligns signals from cluster $h$, while incoherent signals from other clusters are suppressed, resulting in a coherent estimate of the desired cluster update.

The $n$-th symbol becomes:
\begingroup
\allowdisplaybreaks
\begin{align}
&y_{PS}^{n,(h)}(t) = \underbrace{\sum_{m \in \mathcal{S}_h(t)} \left( \frac{1}{K} \sum_{k=1}^{K} |h_{m,k}^{n}(t)|^2 \right) \Delta \theta_{m}^{n,cx}(t)}_{\text{(1) Signal Term}} \nonumber\\
&+ \underbrace{\sum_{m \in \mathcal{S}_h(t)}\sum_{\substack{m' \in \mathcal{S}_h(t) \\ m' \ne m}} \left( \frac{1}{K} \sum_{k=1}^{K} (h_{m,k}^{n}(t))^* h_{m',k}^{n}(t) \right) \Delta \theta_{m'}^{n,cx}(t)}_{\text{(2) Intra-Cluster Interference}} \nonumber\\
&+ \!\underbrace{\sum_{\substack{g \in [H] \\ g \ne h}} \sum_{\substack{m\in \mathcal{S}_h(t)\\ m'\in \mathcal{S}_g(t)}} \!\!\left( \!\frac{1}{K} \sum_{k=1}^{K} (h_{m,k}^{n}(t))^* h_{m',k}^{n}(t) \!\right) \Delta \theta_{m'}^{n,cx}(t)}_{\text{(3) Inter-Cluster Interference}} \nonumber\\ &+ \underbrace{\sum_{m \in \mathcal{S}_h(t)} \left( \frac{1}{K} \sum_{k=1}^{K} (h_{m,k}^{n}(t))^* z_{PS,k}^{n}(t) \right)}_{\text{(4) Noise Term}}.
\end{align}
\endgroup

The combined signal $y_{PS}^{n,(h)}(t)$ is similar to that in \eqref{rec_signal}, consisting of the desired signal term, interference terms, and the noise term. However, due to clustering and the use of partial CSI in the cluster-wise weighted combining (CWC) approach, the interference term consists of two components: (1) intra-cluster interference, arising from imperfect alignment among users within the same cluster, which diminishes as the number of antennas $K \to \infty$; and (2) inter-cluster interference, caused by signal leakage from other clusters due to overlapping channels, which also vanishes under i.i.d. channel assumptions and large $K$.

Ignoring the interference and noise terms in the combined signal, one can use the signal term to estimate the desired cluster average as
\begin{equation}
    \sum_{m \in \mathcal{S}_h(t)} \left( \frac{1}{K} \sum_{k=1}^{K} |h_{m,k}^{n}(t)|^2 \right) \Delta\theta_{m}^{n,\text{cx}}(t),
\end{equation}
which aggregates the users' updates in cluster $h$ along with their corresponding effective channel gains.
The PS then recovers the $n$-th component of the update for cluster $h$ as:
\begin{align}
    \Delta \hat{\boldsymbol{\theta}}_h^n(t) &= \frac{1}{|\mathcal{S}_h(t)| \sigma_h^2} \, \text{Re}\left\{ y_{PS}^{n,(h)}(t) \right\}, \\
    \Delta \hat{\boldsymbol{\theta}}_h^{n+N}(t) &= \frac{1}{|\mathcal{S}_h(t)| \sigma_h^2} \, \text{Im}\left\{ y_{PS}^{n,(h)}(t) \right\},
\end{align}
where $\sigma_h^2$ is the average per-user channel power for cluster $h$. Finally, the PS updates the model of cluster $h$ by:
\begin{equation}
    \boldsymbol{\theta}_h(t+1) = \boldsymbol{\theta}_h(t) + \Delta \hat{\boldsymbol{\theta}}_h(t).
\end{equation}

This combining strategy is closely related to the one introduced in Section~\ref{sec2}, particularly in ~\eqref{combined_signal}, where the global model is recovered by summing all user channel gains across the entire network. In contrast, the current method focuses on combining at the cluster level, using only the sum of the channel gains within a specific cluster. While the core idea remains similar, leveraging conjugate channel responses for coherent combining, the interference characteristics differ. Specifically, the current approach introduces both intra-cluster interference due to misalignment among users within the same cluster and inter-cluster interference from overlapping channels between users in different clusters. 
Although learning performance is expected to decrease slightly due to additional intra-cluster interference, our proposed approach enables training multiple personalized models based on user data characteristics and preferences. Hence, despite this additional impairment, the approach can help users achieve improved learning performance for their specific applications.

% ------------------------------------------------------------
% ------------------------------------------------------------
% ------------------------------------------------------------

% ------------------------------------------------------------
% ------------------------------------------------------------
% ------------------------------------------------------------

\section{\MakeUppercase{Numerical Results}}
\label{sec5}

In this section, we present numerical results for the proposed unified framework, which comprises two complementary components: a cluster-aware scheduling approach for fair global model training and a personalized clustered federated learning scheme for OTA-FL with energy-harvesting mobile devices under highly heterogeneous data distributions.

\subsection{\MakeUppercase{Performance of Scheduling Strategies for Global FL}}

We evaluate the performance of our proposed user scheduling methods across multiple scenarios. We consider image classification tasks on the MNIST \cite{deng2012mnist}, FMNIST \cite{xiao2017fashionmnistnovelimagedataset}, and CIFAR-10 \cite{cifar10} datasets under non-i.i.d. data distributions. 
For MNIST and FMNIST, we use a single-layer network with 784 inputs and 10 outputs ($2N = 7850$); for CIFAR-10, we adopt a CNN ($2N = 797962$) as in \cite{acar2021federated}.
Training is performed using SGD with a learning rate of 0.05, a learning-rate scheduler, $\tau = 5$, and a mini-batch size of $|\xi_m(t)| = 100$ for MNIST and FMNIST, and $\tau = 3$ and $|\xi_m(t)| = 128$ for CIFAR-10.

To simulate highly non-i.i.d. data, we consider two different distribution scenarios. In the first scenario, users' data is limited to a fixed number of labels, either 1 or 2 classes assigned per user. In the second one, we use the Dirichlet distribution to sample $\boldsymbol{p_m} \sim \text{Dir}_{N_c}(\beta)$ with $ \boldsymbol{p_m} = [p_{m,0}, \cdots, p_{m,N_c-1}]$, and user  $m$  receives  $p_{m,{n_c}}$ portion of its data from class  $n_c \in [N_c]$. $\beta$ is a Dirichlet distribution parameter, where smaller values of $\beta$ lead to more unbalanced partitions. We evaluate our setup on highly non-i.i.d. data against a no-scheduling baseline, in which users participate whenever they have sufficient energy.
Throughout the simulations, users are connected to a PS via a wireless fading MAC, in which the channel gains from each user to each PS antenna are i.i.d. The selected parameters are $K = 200$, $\sigma_{h}^{2} = 1$, and $\sigma_{z}^{2} = 0.1$.

% \begin{figure}
%     \centering
%     % First figure
%     \begin{subfigure}{0.48\columnwidth}
%         \centering
%         \includegraphics[trim={0.4cm 0.0cm 0.37cm 0.3cm},clip,width=\textwidth]{figures/ch3/upd_ws1_ci_dir01_100_tnr.pdf}
%         \vspace{-0.6cm}
%         \captionsetup{font=scriptsize}
%         \caption{$\beta = 0.1$.}
%         \label{fig:ci100_dir01}
%     \end{subfigure}
%     % Second figure
%     \begin{subfigure}{0.48\columnwidth}
%         \centering
%         \includegraphics[trim={0.4cm 0.0cm 0.37cm 0.3cm},clip,width=\textwidth]{figures/ch3/upd_ws1_ci_dir02_100_tnr.pdf}
%         \vspace{-0.6cm}
%         \captionsetup{font=scriptsize}
%         \caption{$\beta = 0.2$.}
%         \label{fig:ci100_dir02}
%     \end{subfigure}
%     \vspace{-0.15cm}
%     \captionsetup{font=footnotesize}
%     \caption{Test accuracy of entropy-based scheduling for CIFAR-10 with $M = 100$,
%     $\left| \mathcal{B}_{m} \right| = 500$ and $p_e^{m}(t) = 0.1$.}
%     \label{fig:fig_ci100_dir0102}
%     % \vspace{-0.7cm}
% \end{figure}

\begin{figure}[t]
\centering

\subfloat[$\beta=0.1$.]{
\includegraphics[trim={0.4cm 0.0cm 0.37cm 0.3cm},clip,width=0.47\columnwidth]{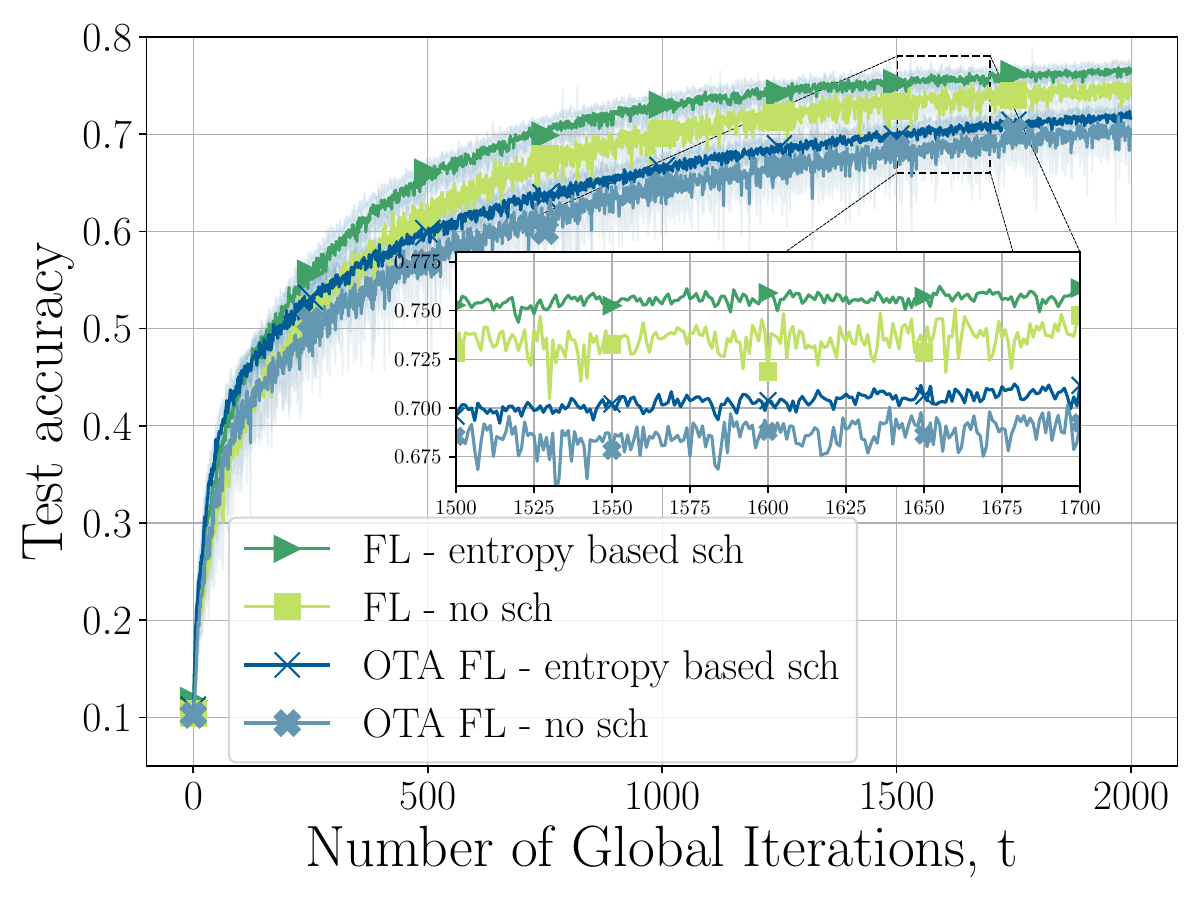}
\label{fig:ci100_dir01}
}
\hfill
\subfloat[$\beta=0.2$.]{
\includegraphics[trim={0.4cm 0.0cm 0.37cm 0.3cm},clip,width=0.47\columnwidth]{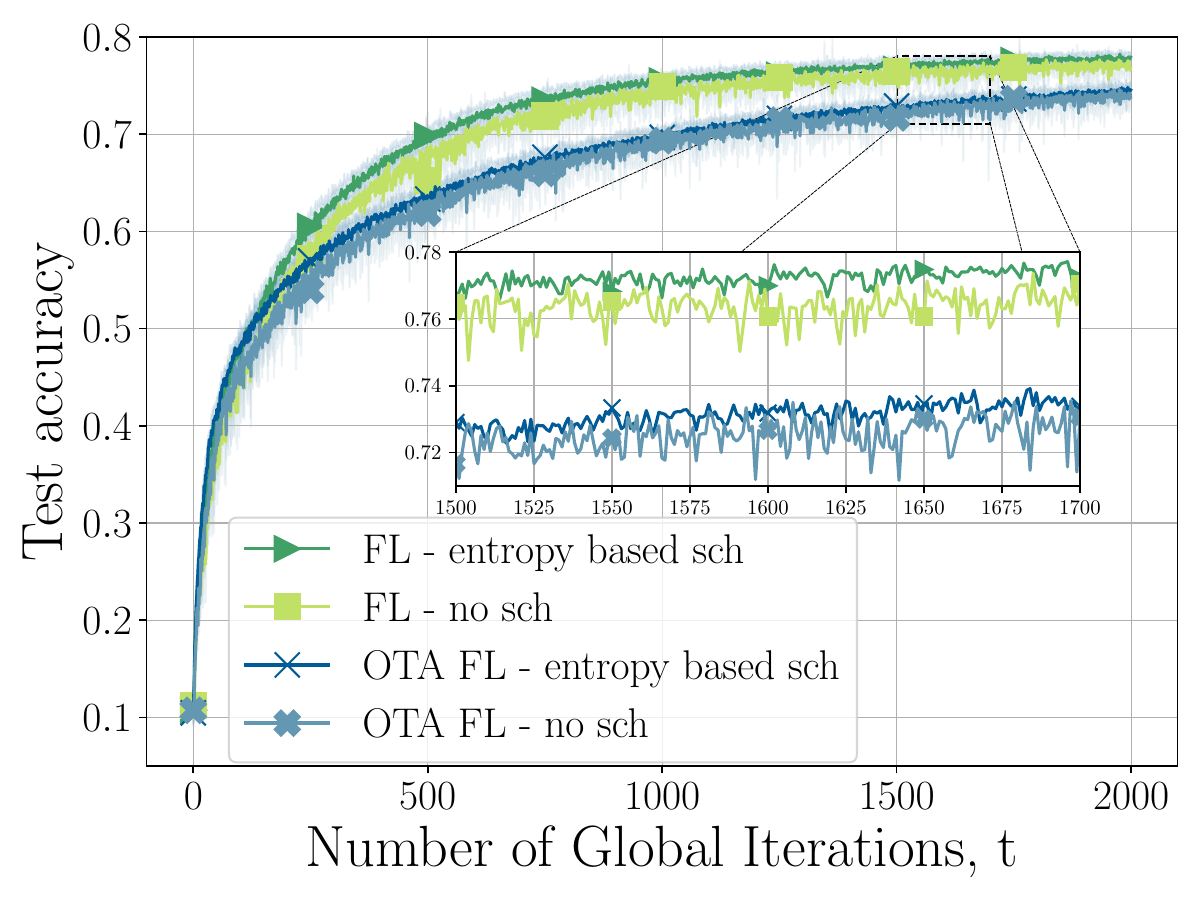}
\label{fig:ci100_dir02}
}

\vspace{-0.15cm}
\caption{Test accuracy of entropy-based scheduling for CIFAR-10 with $M=100$,
$\left|\mathcal{B}_{m}\right|=500$, and $p_e^{m}(t)=0.1$.}
\label{fig:fig_ci100_dir0102}
\end{figure}

In Fig. \ref{fig:fig_ci100_dir0102}, we demonstrate the performance of entropy-based scheduling for CIFAR-10 with $\beta \in \{0.1,0.2\}$  for $M = 100$ users with $\left| \mathcal{B}_{m} \right| = 500$ and $p_e^{m}(t) = 0.1$ for $m \in [M]$. We observe that the gains of our scheme are greater in scenarios with greater heterogeneity, in both the error-free and OTA FL cases. As the distribution becomes more heterogeneous (e.g., $ \beta=0.1$ in Fig. \ref{fig:ci100_dir01}), the impact of the proposed entropy-based scheduling increases. These plots demonstrate that diverse user selection, which yields a more balanced distribution of aggregated data, leads to higher accuracy. The improvements are evident in both error-free FL and OTA FL setups, highlighting the effectiveness of diversity-aware scheduling in mitigating the effects of data heterogeneity.

% \begin{figure}
%     \centering
%     % First figure
%     % \vspace{-7pt}
%     \begin{subfigure}{0.48\columnwidth}
%         \centering
%         \includegraphics[trim={0.4cm 0.2cm 0.37cm 0.3cm},clip,width=\textwidth]{figures/ch3/upd_ws1_mn40_niid1_100_tnr.pdf}
%         \vspace{-0.6cm}
%         \captionsetup{font=scriptsize}
%         \caption{1 class per user and $T=100$.}
%         \label{fig:mn40_niid1}
%     \end{subfigure}
%     % Second figure
%     \begin{subfigure}{0.48\columnwidth}
%         \centering
%         \includegraphics[trim={0.4cm 0.2cm 0.37cm 0.3cm},clip,width=\textwidth]{figures/ch3/upd_ws1_mn40_niid2_100_tnr.pdf}
%         \vspace{-0.6cm}
%         \captionsetup{font=scriptsize}
%         \caption{2 class per user and $T=200$.}
%         \label{fig:mn40_niid2}
%     \end{subfigure}
%     \vspace{-0.15cm}
%     \captionsetup{font=footnotesize}
%     \caption{Test accuracy for MNIST with $M \!=\! 40$, $\left| \mathcal{B}_{m} \right| \!=\! 1250$, $p_e^{m}(t) \!=\! 0.25$.}
%     \label{fig:fig_mn40_niid12}
%     % \vspace{-0.35cm}
% \end{figure}

\begin{figure}[t]
\centering

\subfloat[1 class per user and $T=100$.]{
\includegraphics[trim={0.4cm 0.2cm 0.37cm 0.3cm},clip,width=0.47\columnwidth]
{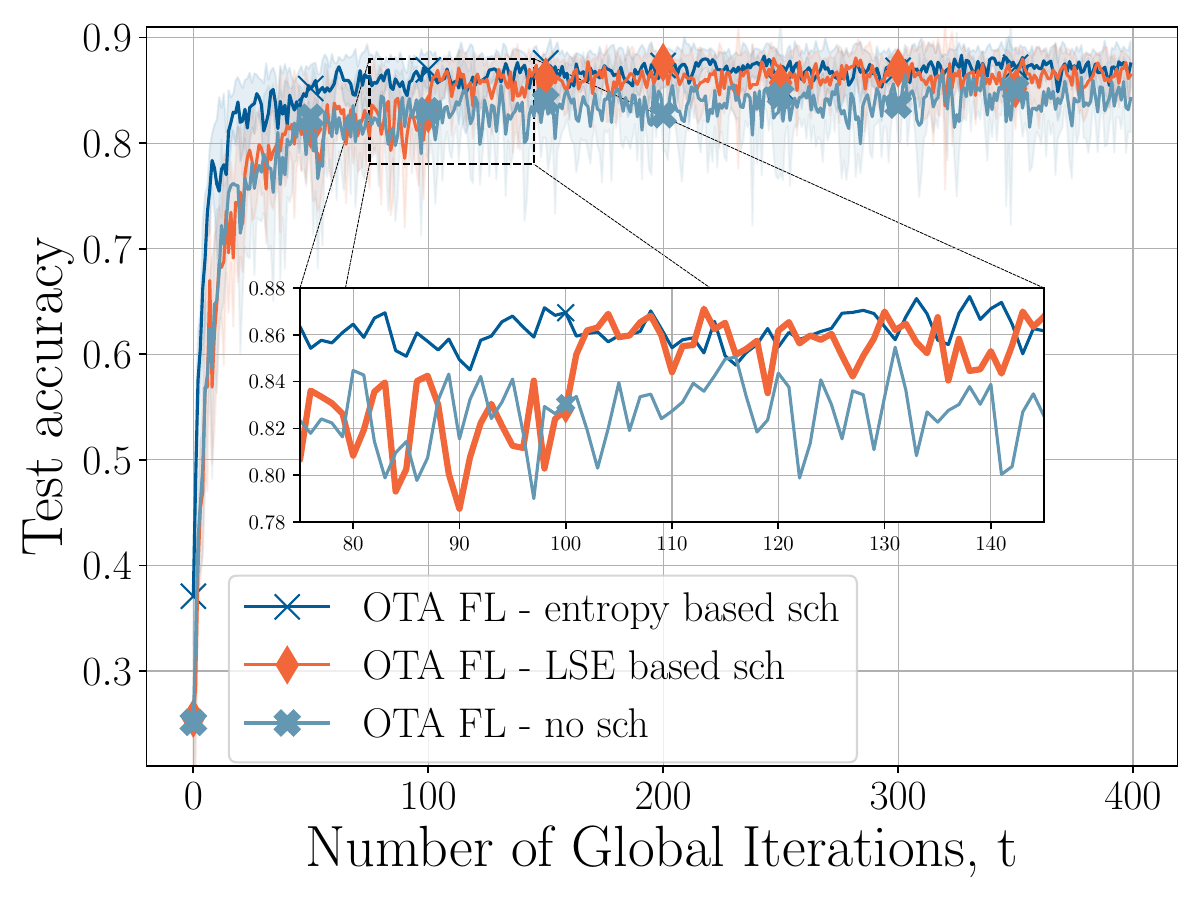}
\label{fig:mn40_niid1}
}
\hfill
\subfloat[2 classes per user and $T=200$.]{
\includegraphics[trim={0.4cm 0.2cm 0.37cm 0.3cm},clip,width=0.47\columnwidth]
{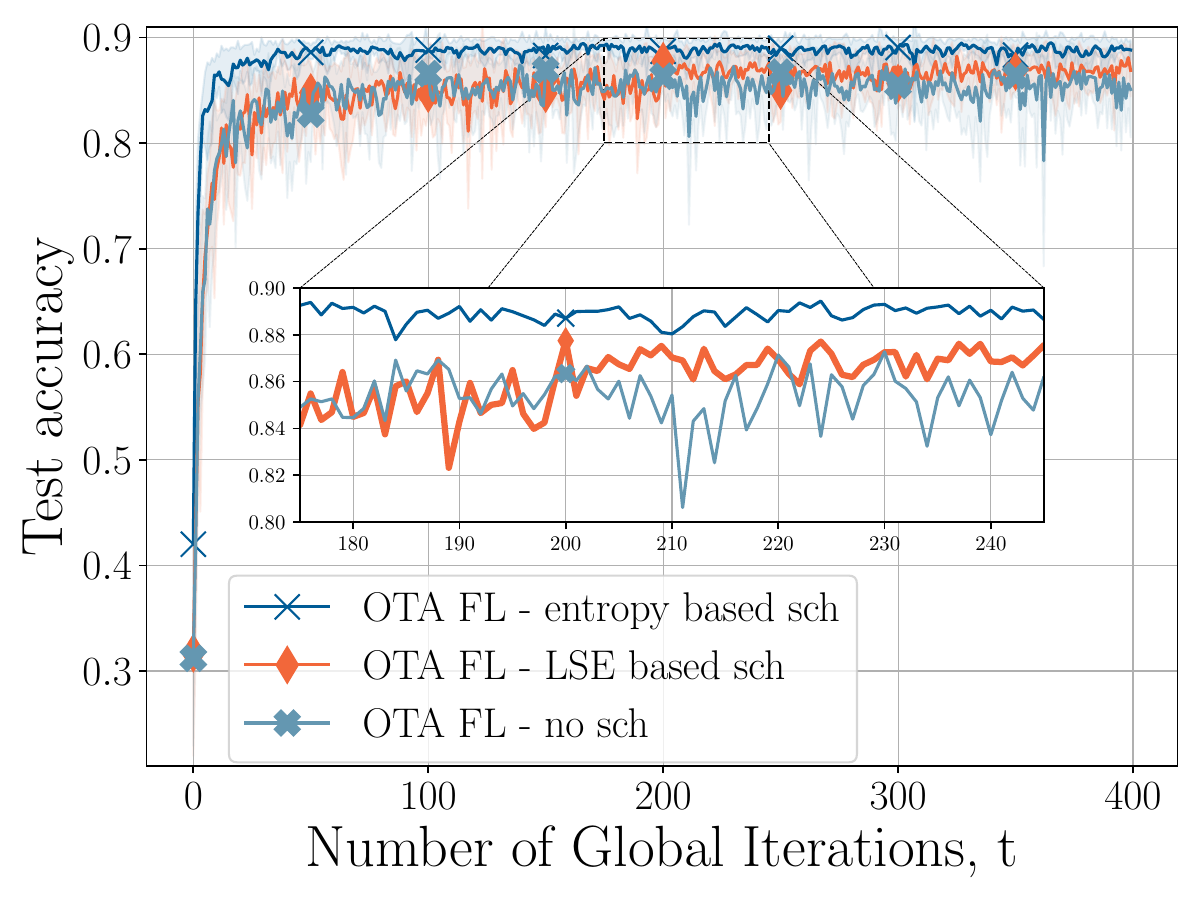}
\label{fig:mn40_niid2}
}

\vspace{-0.15cm}
\caption{Test accuracy for MNIST with $M=40$, $\left|\mathcal{B}_{m}\right|=1250$, and $p_e^{m}(t)=0.25$.}
\label{fig:fig_mn40_niid12}
\end{figure}

Fig. \ref{fig:fig_mn40_niid12} shows the mean test accuracies when the data distributions are not known at the PS. We consider scenarios with one class or two classes per user on the MNIST dataset, with $M = 40$, $\left| \mathcal{B}_{m} \right| = 1250$, $p_e^{m}(t) = 0.25$ for $m \in [M]$ and estimation phases of $T=100$ and $T=200$ iterations. In both cases, entropy-based scheduling yields higher and more stable accuracy. For cases with unknown data distributions, the PS estimates local user representations after $T$ iterations and groups users into 10 clusters, scheduling one user per cluster each iteration. 
As in the previous case, our scheme achieves greater gains in more heterogeneous scenarios. Notably, although estimation and clustering become more difficult in less heterogeneous cases, the proposed scheduling approach still consistently outperforms the no-scheduling baseline.

% \begin{figure}[t]
%     \centering
%     % First figure
%     \begin{subfigure}{0.48\columnwidth}
%         \centering
%         \includegraphics[trim={0.4cm 0.2cm 0.37cm 0.3cm},clip,width=\textwidth]{figures/ch3/upd_ws1_mn20_niid1_100_tnr.pdf}
%         \vspace{-0.6cm}
%         \captionsetup{font=scriptsize}
%         \caption{MNIST, $M = 20$, 1 class per user and $T=100$.}
%         \label{fig:mn20_niid1}
%     \end{subfigure}
%     % % Second figure
%     \begin{subfigure}{0.48\columnwidth}
%         \centering
%         \includegraphics[trim={0.4cm 0.2cm 0.37cm 0.3cm},clip,width=\textwidth]{figures/ch3/upd_ws1_fas40_niid1_200_tnr.pdf}
%         \vspace{-0.6cm}
%         \captionsetup{font=scriptsize}
%         \caption{FMNIST, $M = 40$, 1 class per user and $T=200$.}
%         % \caption{1 class per user and $T=200$.}
        
%         \label{fig:fas40_niid1}
%     \end{subfigure}
%     \vspace{-0.15cm}
%     \captionsetup{font=footnotesize}
%     \caption{Test accuracy for MNIST and FMNIST.}
%     % \caption{Test accuracy for FMNIST with $M \!=\! 40$, $\left| \mathcal{B}_{m} \right| \!=\! 1250$, $p_e^{m}(t) \!=\! 0.25$. }
%     \label{fig:fig_mn20_fas40}
%     % \vspace{-0.7cm}
% \end{figure}

\begin{figure}[t]
\centering

\subfloat[MNIST, $M=20$, 1 class per user and $T=100$.]{
\includegraphics[trim={0.4cm 0.2cm 0.37cm 0.3cm},clip,width=0.47\columnwidth]
{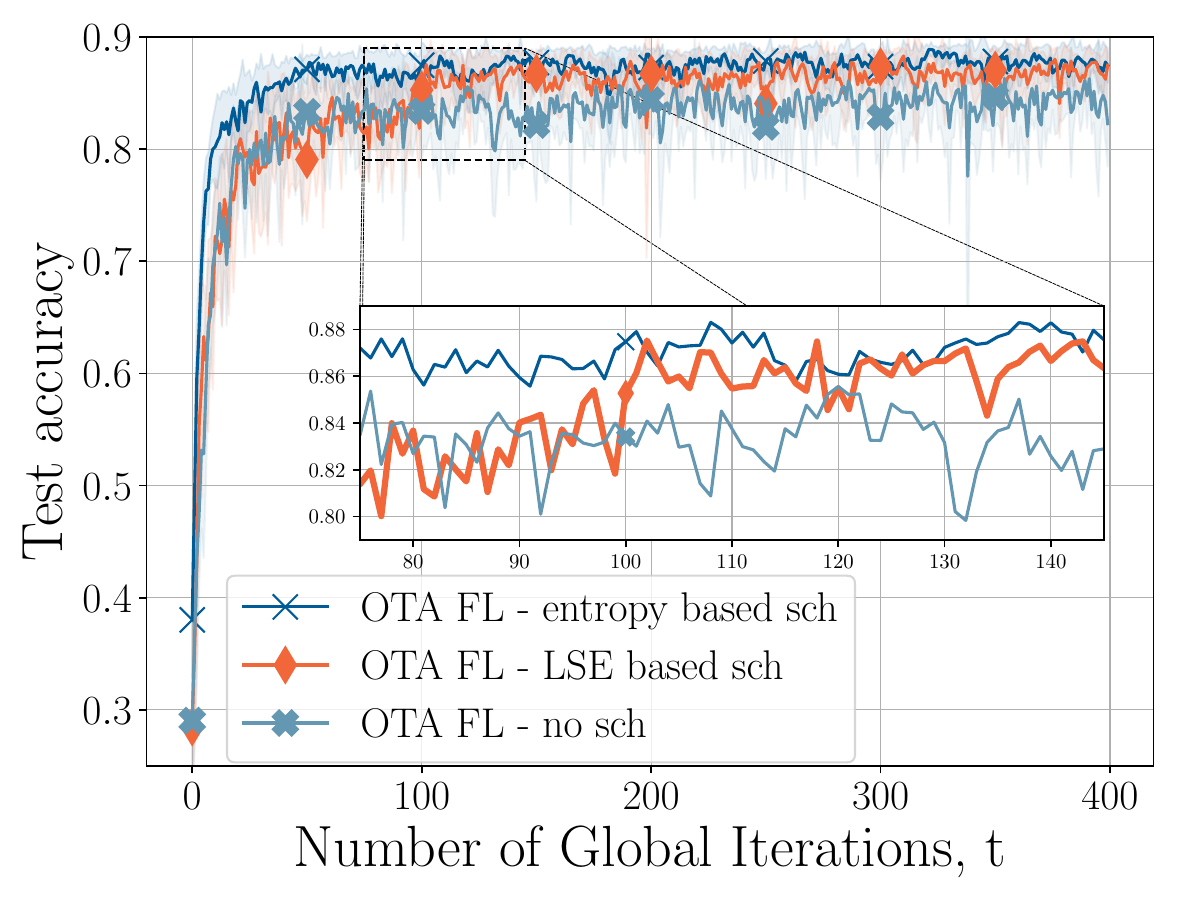}
\label{fig:mn20_niid1}
}
\hfill
\subfloat[FMNIST, $M=40$, 1 class per user and $T=200$.]{
\includegraphics[trim={0.4cm 0.2cm 0.37cm 0.3cm},clip,width=0.47\columnwidth]
{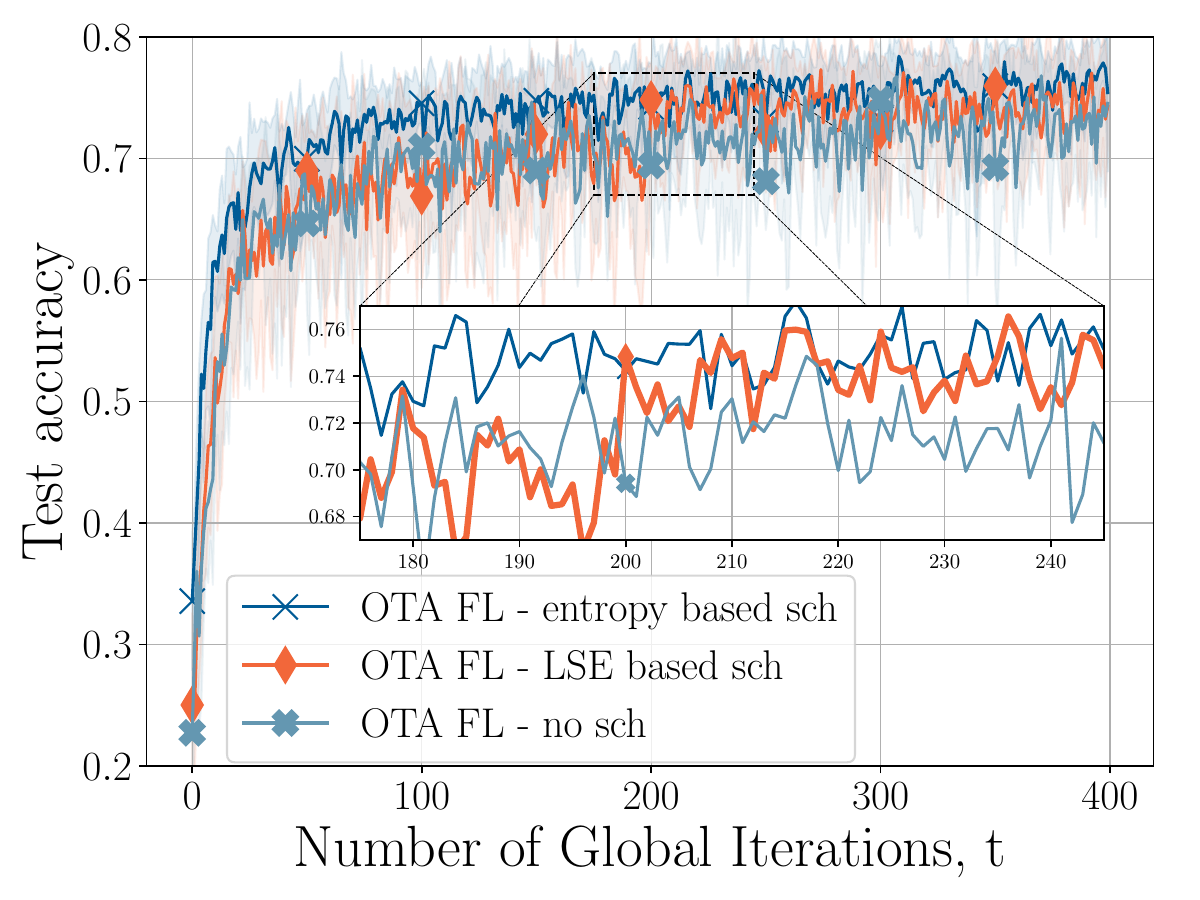}
\label{fig:fas40_niid1}
}

\vspace{-0.15cm}
\caption{Test accuracy for MNIST and FMNIST.}
\label{fig:fig_mn20_fas40}
\end{figure}

% To demonstrate the effectiveness of our proposed method on a different dataset, 
While Fig.~\ref{fig:mn20_niid1} shows trends similar to previous cases, Fig.~\ref{fig:fas40_niid1} presents the performance of our scheduling policies on the FMNIST dataset. In this setup, we consider $M = 40$ users, each with a local dataset size of $\left| \mathcal{B}_{m} \right| = 1250$, and  $p_e^{m}(t) = 0.25$ for $m \in [M]$.  
% Users are divided into 10 scheduling clusters, and one user from each cluster is selected per iteration, following the same approach as in the previous case. 
As in the MNIST setting, entropy-based scheduling outperforms the no-scheduling baseline when data distributions are known. With unknown data distributions, performance improves after the estimation phase and approaches the entropy-based case.

In summary, when data distributions are known, entropy-based scheduling consistently outperforms the no-scheduling baseline by selecting a more representative set of users.
In more realistic settings with unknown data distributions, performance improves after estimation and approaches that of entropy-based scheduling. These results demonstrate that diversity-aware scheduling improves OTA FL for EH devices by mitigating bias due to partial participation.
Moreover, least-squares–based user-representation estimation enables effective scheduling using only aggregated OTA signals, thereby preserving privacy while limiting redundant updates. 
Overall, both entropy- and LSE-based methods better approximate full participation, leading to faster convergence and improved generalization under communication and energy constraints.

\subsection{\MakeUppercase{Performance of Combining Strategies for Clustered FL}}

In this section, we assess the effectiveness of the proposed user-combining strategies for clustered federated learning across various scenarios, aiming to train multiple personalized models tailored to each cluster to better reflect user characteristics. We consider image classification on the MNIST dataset \cite{deng2012mnist} under non-i.i.d. data distributions, using the same setup as in the previous simulations.

We evaluate the three combining methods introduced earlier. In the first setting, MMSE combining with full CSI, the PS has access to complete CSI for each user-to-antenna link, enabling a more precise recovery of individual user updates. In the second setting, MMSE combined with partial CSI, the PS has access only to aggregated CSI at the cluster level, enabling it to recover the combined updates for each cluster. Finally, in the third setting, CWC with partial CSI, the PS employs a beamforming-like approach using the same aggregated cluster-level CSI to extract cluster-specific updates.

To simulate a highly non-i.i.d.\ setting, we adopt a label-partitioning strategy in which each user is assigned data from a single class. For the initial experiments, we employ a simple clustering approach based on class labels, grouping users by their assigned classes. Specifically, we define three clusters ($H = 3$) to evaluate the effectiveness of the proposed methods in a controlled environment. The first cluster contains users with data from classes 0--2, the second from classes 3--6, and the third from the remaining classes.

We compare our cluster-specific methods with two standard FL baselines that train a single global model. In the first baseline, the model is trained using individual user updates recovered under full CSI, similar to MMSE combining with full CSI.
Instead of maintaining separate cluster models, all recovered user updates are averaged to update a single global model. In the second baseline, following Section~\ref{sec3}, PS combines signals from the $K$ antennas by summing the overall channel gains to produce a single global model. 
For Figs. \ref{fig:M20K20} and \ref{fig:M20K5vsK20}, each cluster-specific model is evaluated on the test samples whose labels belong to the corresponding cluster, and the single-global-model baselines are evaluated on the same cluster-specific test sets for a fair comparison.

% \begin{figure*}[t]
%     \centering
%     % First figure
%     \begin{subfigure}{0.5\columnwidth}
%         \centering
%         \includegraphics[clip,width=\textwidth]{figures/ch4/M20K20new1.pdf}
%         \vspace{-0.6cm}
%         \captionsetup{font=scriptsize}
%         \caption{Cluster 1}
%         \label{fig:M20K20_1}
%     \end{subfigure}
%     % Second figure
%     \begin{subfigure}{0.5\columnwidth}
%         \centering
%         \includegraphics[clip,width=\textwidth]{figures/ch4/M20K20new2.pdf}
%         \vspace{-0.6cm}
%         \captionsetup{font=scriptsize}
%         \caption{Cluster 2}
%         \label{fig:M20K20_2}
%     \end{subfigure}
%     \begin{subfigure}{0.5\columnwidth}
%         \centering
%         \includegraphics[clip,width=\textwidth]{figures/ch4/M20K20new3.pdf}
%         \vspace{-0.6cm}
%         \captionsetup{font=scriptsize}
%         \caption{Cluster 3}
%         \label{fig:M20K20_3}
%     \end{subfigure}
%     \vspace{-0.15cm}
%     \captionsetup{font=footnotesize}
%     \caption{Test accuracy for CFL with $M = 20$, $H=3$ and $K=20$.}
%     \label{fig:M20K20}
%     \vspace{-0.45cm}
% \end{figure*}

\begin{figure*}[t]
\centering

\subfloat[Cluster 1]{
\includegraphics[clip,width=0.30\textwidth]
{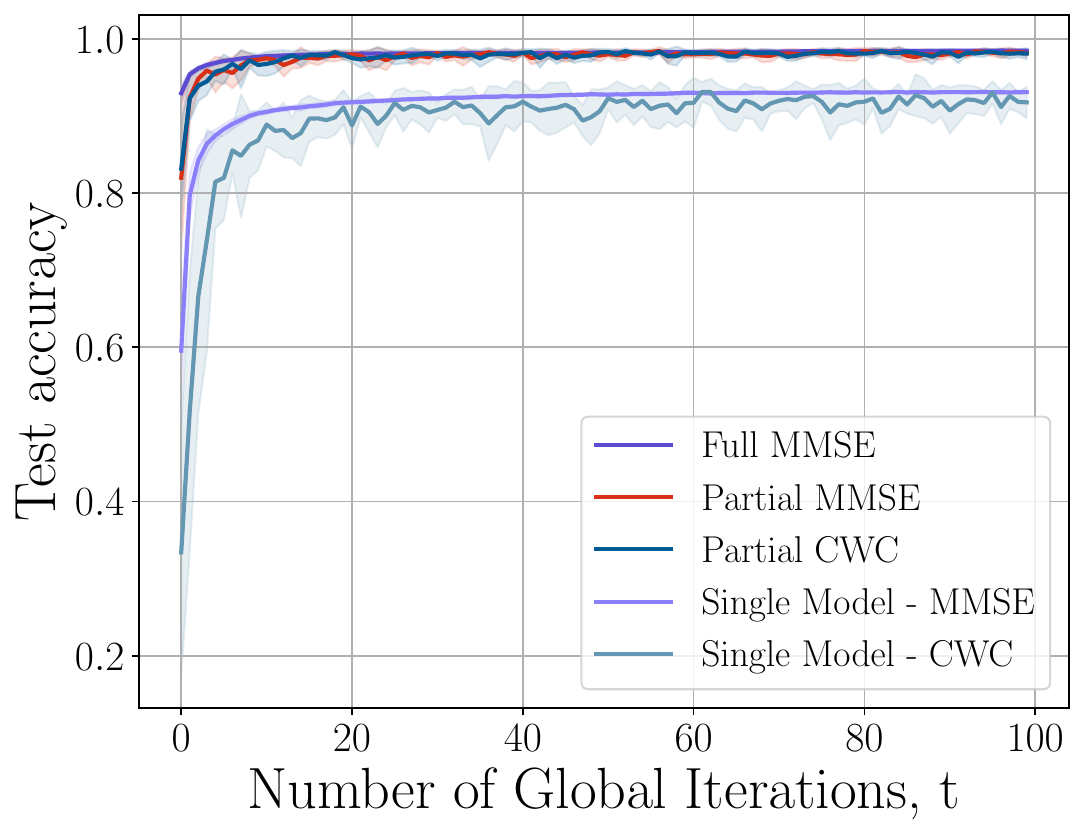}
\label{fig:M20K20_1}
}
\hfill
\subfloat[Cluster 2]{
\includegraphics[clip,width=0.30\textwidth]
{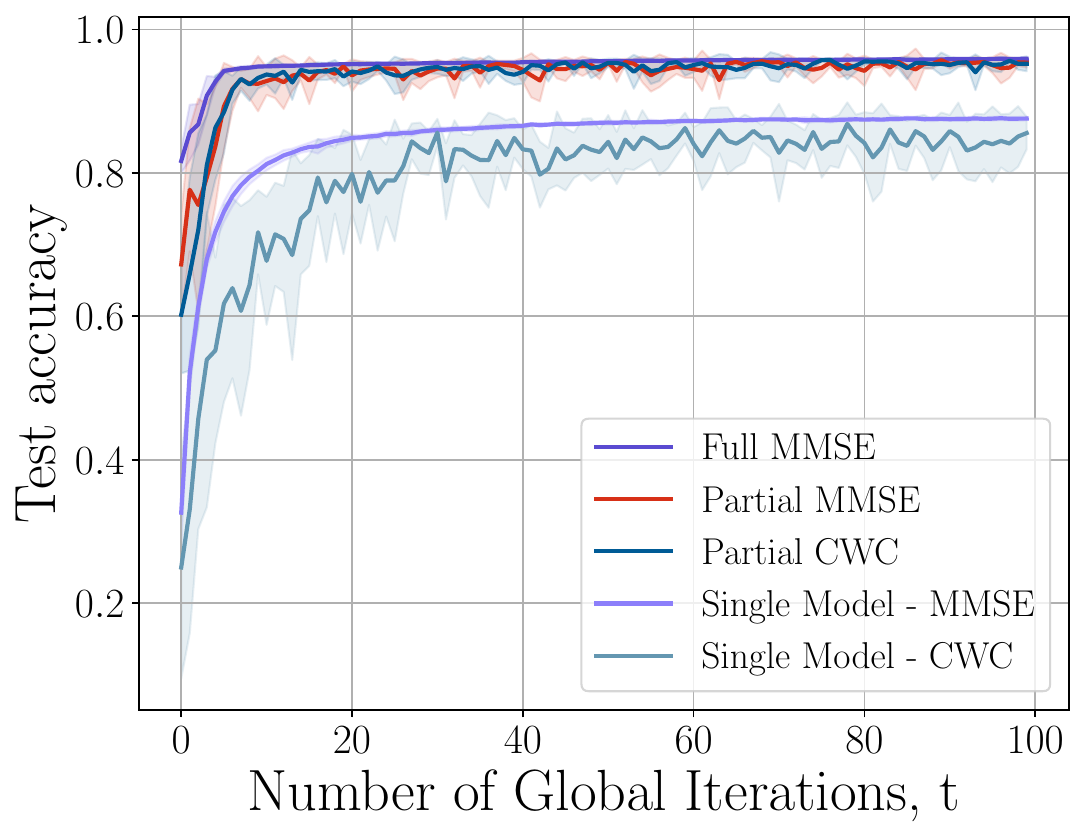}
\label{fig:M20K20_2}
}
\hfill
\subfloat[Cluster 3]{
\includegraphics[clip,width=0.30\textwidth]
{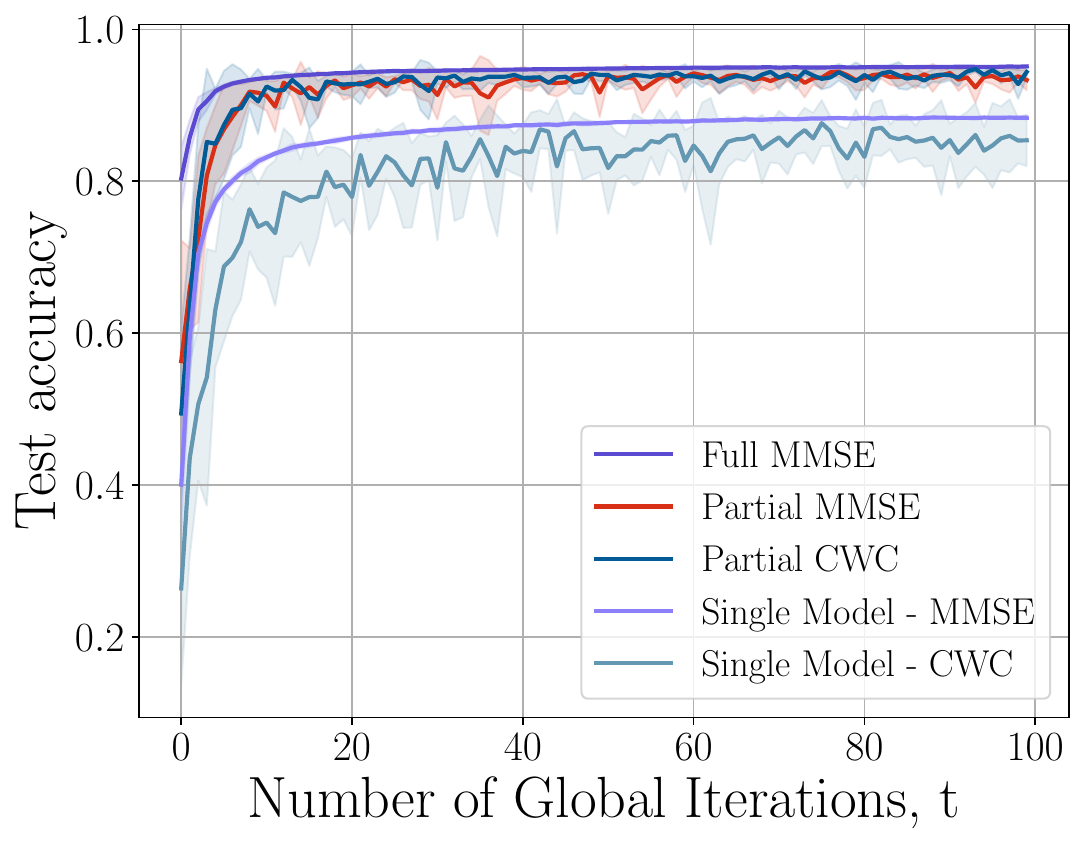}
\label{fig:M20K20_3}
}

\vspace{-0.15cm}
\caption{Test accuracy for CFL with $M=20$, $H=3$, and $K=20$.}
\label{fig:M20K20}
\vspace{-0.45cm}

\end{figure*}

To isolate the combining gain from EH effects, we first consider an ideal case where all users participate in every iteration.
The results under this full‐participation assumption are shown in Fig.~\ref{fig:M20K20}.
We set $M = 20$, $K = 20$, $\sigma_z^2 = 10^{-6}$, and $\sigma_h^2 = 1$. 
We observe that all three cluster-based models, which provide more specialized and personalized global models for users, consistently outperform single-global-model approaches. Notably, even without access to full CSI, both the partial MMSE and cluster-wise weighted combining methods still benefit significantly from clustering. This highlights that leveraging a clustered structure can yield performance gains that surpass even those of a full CSI-based single model. As expected, the full CSI MMSE method outperforms both partial CSI-based approaches, MMSE and CWC, highlighting the performance advantage of having complete channel knowledge. This aligns with our intuition, as access to fine-grained CSI enables more accurate recovery of updates and model aggregation. 

% \begin{figure*}
%     \centering
%     % First figure
%     \begin{subfigure}{0.5\columnwidth}
%         \centering
%         \includegraphics[clip,width=\textwidth]{figures/ch4/M20K5vs20new1.pdf}
%         \vspace{-0.6cm}
%         \captionsetup{font=scriptsize}
%         \caption{Cluster 1}
%         \label{fig:M20K5vsK20_1}
%     \end{subfigure}
%     % Second figure
%     \begin{subfigure}{0.5\columnwidth}
%         \centering
%         \includegraphics[clip,width=\textwidth]{figures/ch4/M20K5vs20new2.pdf}
%         \vspace{-0.6cm}
%         \captionsetup{font=scriptsize}
%         \caption{Cluster 2}
%         \label{fig:M20K5vsK20_2}
%     \end{subfigure}
%     \begin{subfigure}{0.5\columnwidth}
%         \centering
%         \includegraphics[clip,width=\textwidth]{figures/ch4/M20K5vs20new3.pdf}
%         \vspace{-0.6cm}
%         \captionsetup{font=scriptsize}
%         \caption{Cluster 3}
%         \label{fig:M20K5vsK20_3}
%     \end{subfigure}
%     \vspace{-0.15cm}
%     \captionsetup{font=footnotesize}
%     \caption{Test accuracy for CFL with $M = 20$, $H=3$ and $K\in \{5,20\}$.}
%     \label{fig:M20K5vsK20}
%     \vspace{-0.6cm}
% \end{figure*}

\begin{figure*}[t]
\centering

\subfloat[Cluster 1]{
\includegraphics[clip,width=0.30\textwidth]
{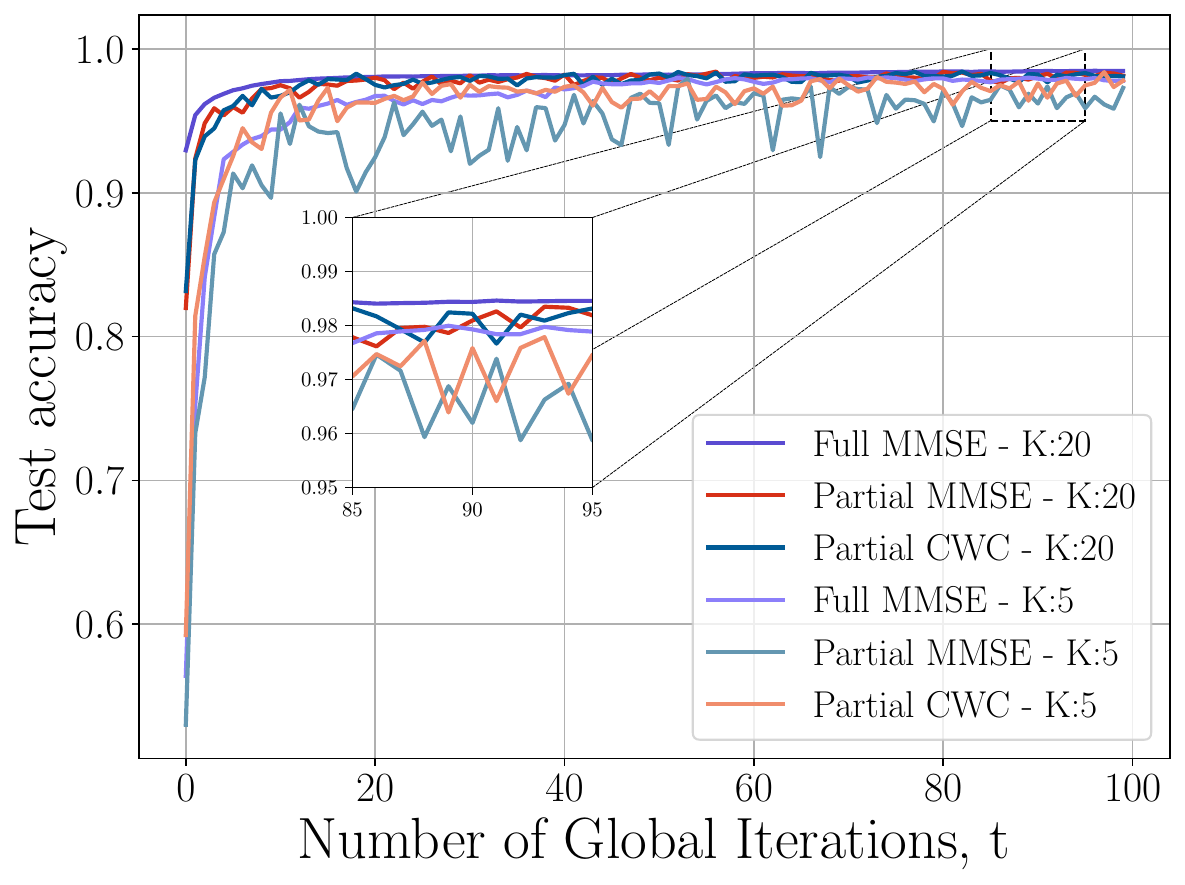}
\label{fig:M20K5vsK20_1}
}
\hfill
\subfloat[Cluster 2]{
\includegraphics[clip,width=0.30\textwidth]
{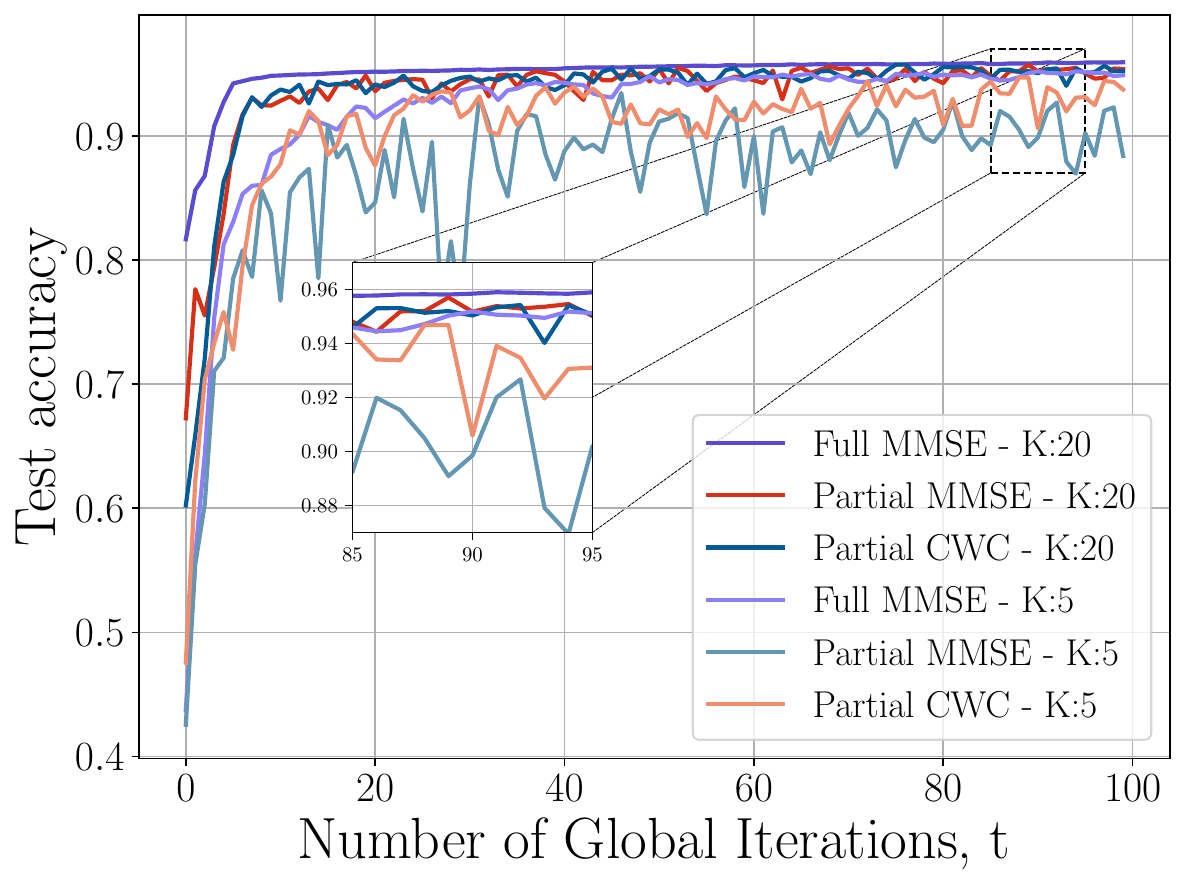}
\label{fig:M20K5vsK20_2}
}
\hfill
\subfloat[Cluster 3]{
\includegraphics[clip,width=0.30\textwidth]
{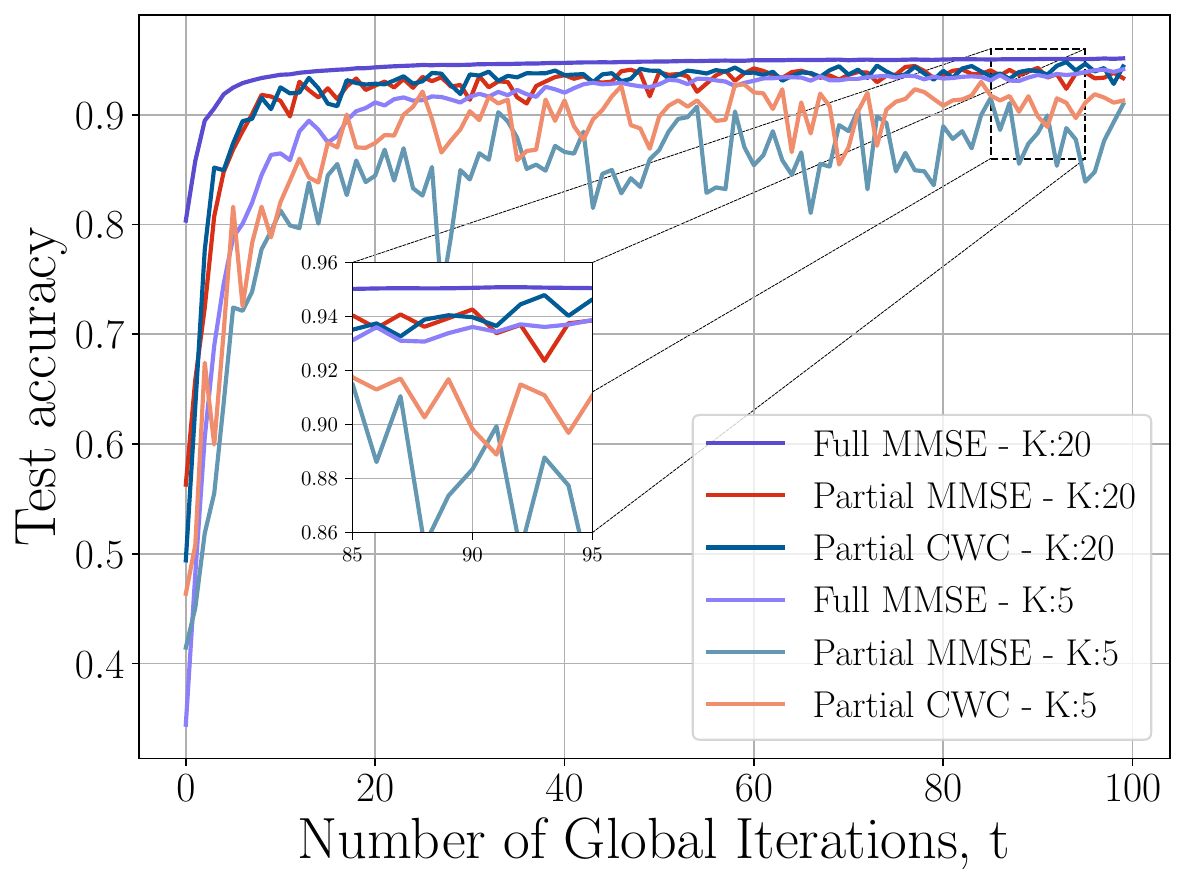}
\label{fig:M20K5vsK20_3}
}

\vspace{-0.15cm}
\caption{Test accuracy for CFL with $M=20$, $H=3$, and $K\in\{5,20\}$.}
\label{fig:M20K5vsK20}
\vspace{-0.6cm}

\end{figure*}

In Fig.~\ref{fig:M20K5vsK20}, we compare the performance of our proposed combining methods under different numbers of receiver antennas. Specifically, we evaluate setups with $K = 5$ and $K = 20$ antennas, using the same channel model as in the previous setup, to examine the effect of antenna count on each combining method's performance. We observe that the full CSI MMSE method exhibits the highest resilience to reduced antenna count, with only a slight performance drop compared to the $K = 20$ case. In contrast, the partial CSI methods (both MMSE and CWC) exhibit notably lower, less stable learning performance under limited-antenna conditions. Nevertheless, they achieve a relatively acceptable accuracy, typically around 5–10 percentage points lower than that of the full CSI MMSE setup with $K = 20$.

Next, we present the performance of our proposed combining methods on the CIFAR-10 dataset \cite{cifar10}. For CIFAR-10, we use a CNN with a total of \( 2N = 61{,}006 \) parameters. The architecture comprises two convolutional layers, each followed by max-pooling, and three fully connected layers. 
We also incorporate EH characteristics for each user, assuming a constant EH probability of \( p_e^{m}(t) = 0.25 \), \( \forall m, t \). We compare our proposed combining methods with the standard FL setup that uses a single global model. In the highly non-i.i.d. setup, we divide users into five categories, each associated with a subset of 4 classes. Each user receives a data distribution drawn from the label subset of its assigned category using a Dirichlet distribution, i.e., \( \boldsymbol{p}_m \sim \text{Dir}_{4}(\beta) \), where \( \boldsymbol{p}_m = [p_{m,0}, \cdots, p_{m,3}] \). Accordingly, user \( m \) receives \( p_{m,n_c} \) portion of its local data from class \( n_c \in [4] \), with the Dirichlet concentration parameter set to \( \beta = 1 \). Note that, for ease of comparison, we report the average cluster-wise accuracy as a single curve.

Fig.~\ref{fig:CFL_cifar} shows that the proposed combining methods successfully demonstrate the feasibility of delivering distinct personalized models to different clusters simultaneously over the air. As expected, the full CSI MMSE method outperforms both the partial MMSE and partial CWC approaches, consistent with earlier results. 
We also observe that the standard FL approach, which learns a single global model for all users, performs significantly worse than the MNIST results shown in Fig.~\ref{fig:M20K20}.
This performance gap likely reflects the increased complexity of the CIFAR-10 dataset, highlighting the importance of deploying more personalized models in such settings.
\begin{figure}[t]
    \centering
    % First figure
    \includegraphics[trim={0 0.3cm 0 0},clip,width=.7\linewidth]{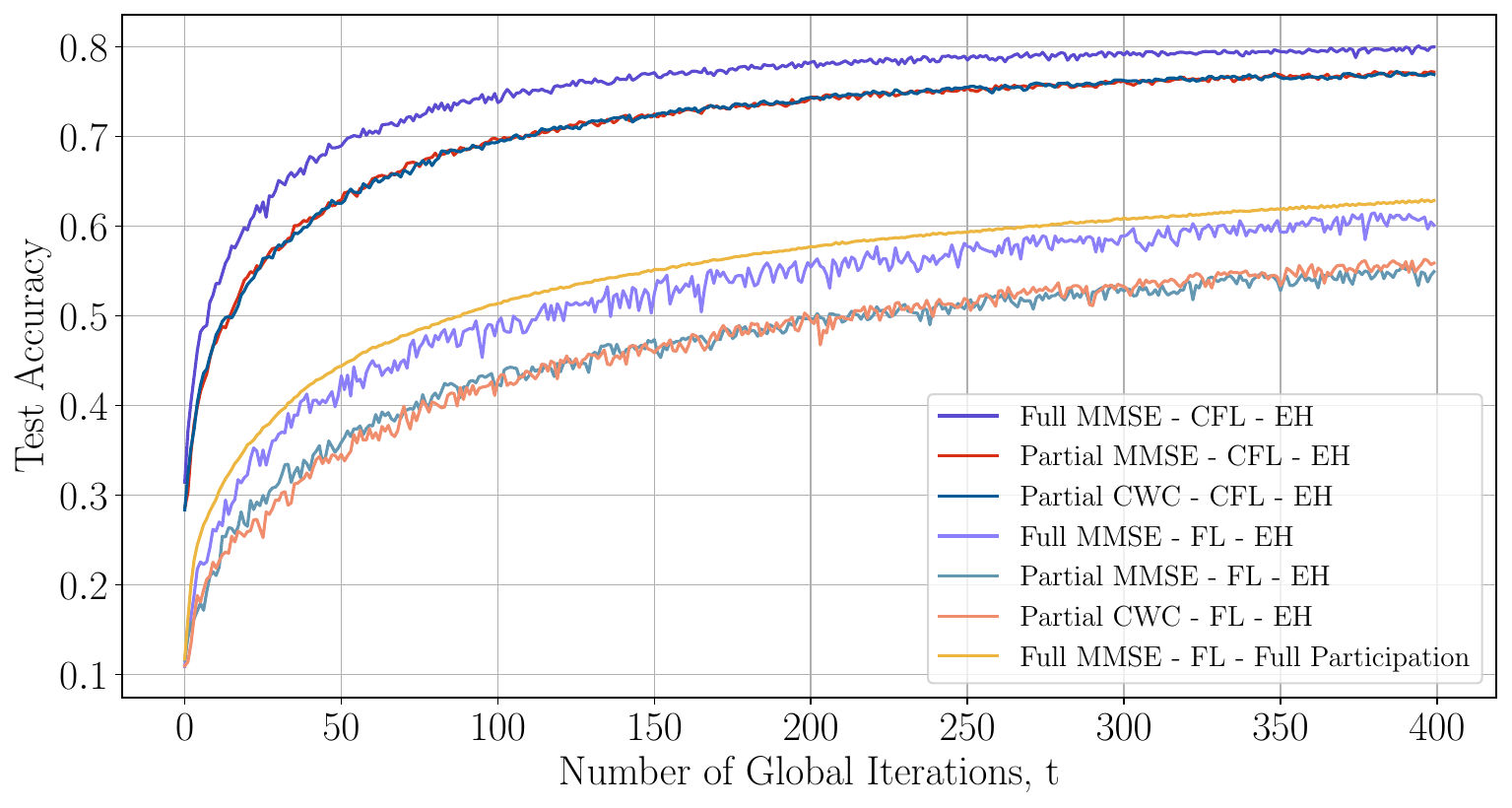}
    % \vspace{-0.2cm}
    %\vspace{-0.15cm}
    \captionsetup{font=footnotesize}
    \caption{Test accuracy for CIFAR-10, CFL with $M = 40$, $H=5$ and $K=80$.}
    \label{fig:CFL_cifar}
    % \vspace{-0.7cm}
\end{figure}

When full CSI is available, the full MMSE method enables accurate recovery of individual updates and achieves the highest performance and robustness, especially with a limited number of antennas. With only partial CSI, partial MMSE and partial CWC provide practical alternatives that extract cluster-level updates to support clustered training. Although their performance is lower than that of full MMSE, they still outperform standard federated learning baselines and enable effective cluster-level personalization under limited CSI. Overall, the results highlight the trade-offs between performance and CSI availability and validate the proposed OTA FL strategies in heterogeneous wireless environments.

\begin{remark}[Comparison of Global and Personalized Learning Paradigms]
The numerical results highlight the complementary nature of the two proposed approaches. When a single model that captures the average characteristics of all users is desired, the cluster-aware global learning framework provides a balanced and fair solution. In contrast, when user-specific performance is prioritized, the personalized clustered learning approach more effectively exploits data heterogeneity to adapt models to individual users' preferences. These two strategies represent opposite ends of the design spectrum, i.e., generalization versus personalization, while both benefit from over-the-air aggregation to achieve communication efficiency under energy-harvesting and wireless constraints. 
\end{remark}

% ------------------------------------------------------------
% ------------------------------------------------------------
% ------------------------------------------------------------

% ------------------------------------------------------------
% ------------------------------------------------------------
% ------------------------------------------------------------

\section{\MakeUppercase{Conclusions}}
\label{conc}
This paper investigates how the inherent clustered structure of heterogeneous users can be exploited to improve federated learning in communication- and energy-constrained wireless networks. 
We first provide a convergence analysis under both decaying and constant stepsizes, characterizing the steady-state error floor and identifying the explicit contributions of OTA channel noise, data heterogeneity, and gradient approximation error to the limiting bound, thereby motivating scheduling.
We then propose cluster-aware, diversity-aware user scheduling strategies that lead to fairer and more representative training when learning a single global model, even under stochastic energy harvesting and OTA aggregation. Specifically, entropy-based scheduling and least-squares–based inference of user representations enable effective client selection for both known and unknown data distributions using only aggregated OTA signals.
We then extend the framework to personalized clustered FL, where separate models are trained for naturally aligned user groups. By developing OTA aggregation and combining methods across different levels of channel state information, the proposed approach enables simultaneous training of multiple personalized models on a single parameter server. These results highlight promising directions for future research on scalable distributed optimization, privacy-preserving inference from aggregated signals, and adaptive communication–learning co-design in large-scale wireless federated systems.

% ------------------------------------------------------------
% ------------------------------------------------------------
% ------------------------------------------------------------

% ------------------------------------------------------------
% ------------------------------------------------------------
% ------------------------------------------------------------

% \section*{ACKNOWLEDGMENT}
% The preferred spelling of the word ``acknowledgment'' in
% American English is without an ``e'' after the ``g.'' Use the
% singular heading even if you have many acknowledgments.
% Avoid expressions such as ``One of us (S.B.A.) would like
% to thank . . . .'' Instead, write ``F. A. Author thanks . . . .'' In
% most cases, sponsor and financial support acknowledgments
% are placed in the unnumbered footnote on the first page, not
% here.

\appendices
\section{Proof of Theorem \ref{thm1}} 
\label{appendixA}
% This appendix outlines the proof of Theorem \ref{thm1}. 
% In this appendix, we provide an outline for the proof of Theorem \ref{thm1}.
We define:
\begingroup
\allowdisplaybreaks
\begin{align} 
    \boldsymbol{w}(t+1) &\triangleq \boldsymbol{\theta}_{PS}(t) + \frac{1}{\left| \mathcal{S}(t) \right| }\sum_{m\in \mathcal{S}(t)}\Delta \boldsymbol{\theta}_{m}(t), \label{aux_w} \\
    \boldsymbol{v}(t+1) &\triangleq \boldsymbol{\theta}_{PS}(t) + \frac{1}{M} \sum_{m=1}^{M} \Delta \boldsymbol{\theta}_{m}(t). \label{aux_v}
\end{align}
\endgroup
From (\ref{gl_update_part}), we have $    \boldsymbol{\theta}_{PS}(t+1)=\boldsymbol{\theta}_{PS}(t)+ \Delta \hat{\boldsymbol{\theta}}_{PS}(t)$. Using this, we can derive 
\begingroup
\allowdisplaybreaks
\begin{align} \label{app1}
&\hspace {-2pc}\left \| \boldsymbol{\theta}_{PS}(t+1) - \boldsymbol{\theta}^* \right \|_{2}^{2} \notag\\ 
=& \left \| \boldsymbol{\theta}_{PS}(t+1) - \boldsymbol{w}(t+1) + \boldsymbol{w}(t+1) - \boldsymbol{\theta}^* \right \|_{2}^{2} \notag\\ 
=& \left \| \boldsymbol{\theta}_{PS}(t+1) - \boldsymbol{w}(t+1) \right \|_{2}^{2} + \left \| \boldsymbol{w}(t+1) - \boldsymbol{\theta}^* \right \|_{2}^{2} \notag\\ 
&+\, 2\langle \boldsymbol{\theta}_{PS}(t+1) - \boldsymbol{w}(t+1), \boldsymbol{w}(t+1) - \boldsymbol{\theta}^* \rangle.
\end{align}
\endgroup
To bound these terms, we employ the following lemmas.

\begin{lemma} \label{lemma1}
For the first and third terms in (\ref{app1}), we have
{\small{
    \begin{align*}&\mathbb {E} \left [{ \left \|{ \boldsymbol {\theta }_{PS}(t+1) - {\boldsymbol{w}} (t+1) }\right \|_{2}^{2} }\right] \le \frac { \eta ^{2}(t) \tau ^{2} G^{2}}{K} + \frac {\sigma _{z}^{2}N}{ K {\left| \mathcal{S}(t) \right|} \sigma _{h}^{2}} \label{bound_term1} ,
    \end{align*}}}
   %and
    \begin{equation*} \mathbb {E} \big [\langle \boldsymbol{\theta}_{PS}(t+1) - \boldsymbol{w}(t+1), \boldsymbol{w}(t+1) - \boldsymbol{\theta}^* \rangle \big] = 0.\end{equation*}
\end{lemma}
\begin{proof}
    The proofs are similar to Lemmas 1 and 3 in \cite{amiri2020BFL}.
\end{proof}
For the second term in (\ref{app1}), we proceed as follows:
\begingroup
\allowdisplaybreaks
\begin{align} \label{app_rec}
&\hspace {-2pc}\left \| \boldsymbol{w}(t+1) - \boldsymbol{\theta}^* \right \|_{2}^{2} \notag\\ 
=& \left \| \boldsymbol{w}(t+1) - \boldsymbol{v}(t+1) + \boldsymbol{v}(t+1) - \boldsymbol{\theta}^* \right \|_{2}^{2} \notag\\ 
=& \left \| \boldsymbol{w}(t+1) - \boldsymbol{v}(t+1) \right \|_{2}^{2} + \left \| \boldsymbol{v}(t+1) - \boldsymbol{\theta}^* \right \|_{2}^{2} \notag\\ 
&+\, 2\langle \boldsymbol{w}(t+1) - \boldsymbol{v}(t+1), \boldsymbol{v}(t+1) - \boldsymbol{\theta}^* \rangle.
\end{align}
\endgroup
\begingroup
\allowdisplaybreaks
\begin{lemma} \label{lemma2}
    For the second term in (\ref{app_rec}), we have
    \begingroup
    \allowdisplaybreaks
    \begin{align}
    & \mathbb {E} \left [{ \left \|{ \boldsymbol{v} (t+1) - {\boldsymbol {\theta }}^{*} }\right \|_{2}^{2} }\right] \notag\\
 & \qquad  \le\left ({1 - \mu \eta (t)~\left ({\tau - \eta (t) (\tau - 1) }\right) }\right) \mathbb {E} \left [{ \left \|{ \boldsymbol {\theta }_{PS}(t)\,\,- {\boldsymbol {\theta }}^{*} }\right \|_{2}^{2} }\right] \notag\\
    & \qquad \qquad + \left ({1+ \mu (1- \eta (t)) }\right) \eta ^{2}(t) G^{2} \frac {\tau (\tau -1)(2\tau -1)}{6} \notag\\
    & \qquad \qquad+\, \eta ^{2}(t) (\tau ^{2} + \tau -1) G^{2} + 2 \eta (t) (\tau - 1) \Gamma  \label{app2}.
    \end{align}
    \endgroup
\end{lemma}
\endgroup
\begin{proof} 
The proof follows from ~\cite[Lemma~2]{amiri2020BFL}.
\end{proof}
\begin{lemma} \label{lemma3}
For the first term in (\ref{app_rec}), we have
\begin{align*}
    &\mathbb{E} \left[ \left\| \boldsymbol{w}(t+1) - \boldsymbol{v}(t+1) \right\|_{2}^{2} \right] = \left( \eta^2(t) \tau(\tau-1)LG + \eta(t)\tau \epsilon  \right)^{2}.
\end{align*}
\end{lemma}
\begin{proof}
    The proof is similar to \cite[Lemma 1]{balakrishnan2022diverse}.
\end{proof}
\begin{lemma} \label{lemma4}
The third term in (\ref{app_rec}) is bounded by
\begin{align}
    &\mathbb{E} \left[     2\langle \boldsymbol{w}(t+1) - \boldsymbol{v}(t+1), \boldsymbol{v}(t+1) - \boldsymbol{\theta}^* \rangle \right] \notag\\
    &\qquad \qquad \leq \left( \eta^2(t) \tau(\tau-1)LG + \eta(t)\tau \epsilon \right) c,
\end{align}
for some constant $c$, which is related to $\Gamma$, $G$  and \(\mu\).
\end{lemma}
\begin{proof}
The proof is similar to \cite[Lemma 1]{balakrishnan2022diverse}.
\end{proof}
By combining Lemmas 1–4, the theorem is proved.

%%%%%%%%%%%%%%%%%%%%%
\section{Proof of Corollary \ref{cor1}}
\label{appendixB}

\textbf{Step 1: Bounding $A(i)$ and establishing contraction.}
Let $\eta(i) = c_0/(i+t_0)$. 
% Write $A(i) = 1 - \mu\eta(i)[\tau - (\tau-1)\eta(i)]$. 
Since $\eta(i) \leq \eta(0)$, we have $\tau - (\tau-1)\eta(i) \geq q_0 > 0$ for all $i$, so $A(i) \leq 1 - \mu q_0\eta(i)$. Using $1-x \leq e^{-x}$ and $\sum_{i=0}^{t-1}c_0/(i+t_0) \geq c_0\ln((t+t_0)/t_0)$, which follows from the integral lower bound $\sum_{i=0}^{t-1}1/(i+t_0) \geq \int_0^t 1/(x+t_0)\,dx$ for the decreasing function $1/(x+t_0)$:
\begin{align}
    \prod_{i=0}^{t-1} A(i)
    &\leq \exp\!\left(-\mu q_0 \sum_{i=0}^{t-1}\eta(i)\right)
    \leq \left(\frac{t_0}{t+t_0}\right)^{\alpha}.
\end{align}
Since $\alpha > 1$, the initial-error term is $o(1/t)$.

\textbf{Step 2: Upper-bounding $B(i)$.}
Using $\eta(i) \leq 1$ and $\epsilon(i) \leq \bar{\epsilon}$, we bound $(\eta^2(i)a + \eta(i)b)^2 \leq \eta^2(i)(a+b)^2$ with $a = \tau(\tau-1)LG$ and $b = \tau\bar{\epsilon}$.
% and $1+\mu(1-\eta(i)) \leq 1+\mu$. 
This gives the upper bound $B(i) \leq B_0 + \eta(i)B_1 + \eta^2(i)B_2$, where:
\begin{align}
    B_0 &\triangleq \frac{\sigma_z^2 N}{K S_{\min} \sigma_h^2}, \label{eq:B0}\\
    B_1 &\triangleq 2(\tau-1)\Gamma + \tau\bar{\epsilon} c, \label{eq:B1}\\
    B_2 &\triangleq \frac{\tau^2 G^2}{K}
          + (1+\mu)G^2\frac{\tau(\tau-1)(2\tau-1)}{6} \notag\\
         &\quad + (\tau^2+\tau-1)G^2
          + \left[\tau(\tau-1)LG + \tau\bar{\epsilon}\right]^2 \notag\\
         &\quad + \tau(\tau-1)LGc. \label{eq:B2}
    % B_2^\infty &\triangleq (1+\mu)G^2\frac{\tau(\tau-1)(2\tau-1)}{6} \notag\\
    %      &\quad + (\tau^2+\tau-1)G^2
    %       + \left(\tau(\tau-1)LG + \tau\bar{\epsilon}\right)^2 \notag\\
    %      &\quad + \tau(\tau-1)LGc, \label{eq:B2inf}
\end{align}
% where $B_2^\infty = \lim_{K\to\infty}B_2$.

\textbf{Step 3: Bounding the weighted sum.}
Substituting into~\eqref{conv1}:
\begin{align}
    &\mathbb{E}\!\left[\|\boldsymbol{\theta}(t)-\boldsymbol{\theta}^*\|_2^2\right]
    \leq \left(\frac{t_0}{t+t_0}\right)^{\alpha}
    \|\boldsymbol{\theta}(0)-\boldsymbol{\theta}^*\|_2^2
    + B_0\,\Phi(t) \notag\\
    &\quad + \sum_{j=0}^{t-1}\!\left(\eta(j)B_1 + \eta^2(j)B_2\right)
    \prod_{i=j+1}^{t-1}\!A(i). \label{eq:step3_split}
\end{align}
To verify that $\Phi(t) = \Theta(t)$, note first that $\Phi(t) \leq t$ since $A(i)$ is at most one. For the lower bound, use $1-x \geq e^{-2x}$ for $x \in (0,1/2)$ to write $\prod_{i=j+1}^{t-1}A(i) \geq \left(\frac{j+1+t_0}{t+t_0}\right)^{2\alpha}$, where the exponent follows from the integral upper bound on the harmonic sum. Summing over $j$ and bounding below by an integral gives $\Phi(t) \geq (t+t_0)/(2\alpha+1) = \Theta(t)$, confirming that the $B_0$ term grows linearly for finite $K$.

For the $B_1$ sum, use the bound $\prod_{i=j+1}^{t-1}A(i) \leq \kappa_\alpha\left(\frac{j+t_0}{t+t_0}\right)^{\alpha}$, where $\kappa_\alpha = (1+1/t_0)^\alpha$ follows from $(j+1+t_0)/(j+t_0) \leq 1+1/t_0$. Using the upper bound $\sum_{j=0}^{t-1}(j+t_0)^{\alpha-1} \leq t_0^{\alpha-1} + (t+t_0)^\alpha/\alpha$, which holds for all $\alpha > 1$ via $\sum_{j=0}^{t-1}f(j) \leq f(0) + \int_0^{t-1}f(x)\,dx$:
\begin{align}
    \sum_{j=0}^{t-1}\eta(j)B_1\prod_{i=j+1}^{t-1}A(i)
    \leq \frac{\kappa_\alpha B_1}{\mu q_0}
    \left[1 - \left(\frac{t_0}{t+t_0}\right)^{\alpha}\right],
    \label{eq:B1bound}
\end{align}
which is bounded above by the constant $\kappa_\alpha B_1/(\mu q_0)$ for all $t$, and converges to this constant from below as $t \to \infty$.

For the $\eta^2(j)B_2$ sum, the same product bound gives:
\begin{align}
    &\sum_{j=0}^{t-1}\eta^2(j)B_2\prod_{i=j+1}^{t-1}A(i) \notag\\
    &\quad \leq \frac{\kappa_\alpha c_0^2 B_2}{(t+t_0)^{\alpha}}
    \left[t_0^{\alpha-2} + \frac{(t+t_0)^{\alpha-1} - t_0^{\alpha-1}}{\alpha-1}\right]
    \triangleq R_{2,K}(t),
    \label{eq:R2K}
\end{align}
where the subscript $K$ reflects the dependence of $B_2$ on $K$. Since the bracket grows as $O((t+t_0)^{\alpha-1})$, dividing by $(t+t_0)^\alpha$ gives $R_{2,K}(t) = O(1/t)$ for each fixed $K$.

Summing these bounds gives us the finite-$K$ upper bound in \eqref{cor1:finiteK}.

\textbf{Step 4: Large-$K$ limit and asymptotic neighborhood.}
As $K \to \infty$, $B_0 \to 0$, $\Phi(t)$ contributes nothing, and $B_2 \to B_2^\infty=\lim_{K\to\infty}B_2=B_2-\frac{\tau^2 G^2}{K}$, so $R_{2,K}(t) \to R_2(t)$, where:
\begin{align}
    R_2(t) &\triangleq \frac{\kappa_\alpha c_0^2 B_2^\infty}{(t+t_0)^{\alpha}}
    \left[t_0^{\alpha-2} + \frac{(t+t_0)^{\alpha-1} - t_0^{\alpha-1}}{\alpha-1}\right]
    \notag\\
    &= O(1/t).
    \label{eq:R2}
\end{align}
The bound reduces to~\eqref{cor1:largeK}. Taking $t \to \infty$, both the initial error term and $R_2(t)$ vanish, leaving only the $B_1$ contribution, which yields~\eqref{cor1:limsup}.

% \vfill
\bibliographystyle{IEEEtran}
\bibliography{ref_ieee.bib}
\end{document}